%% file: sample_paper.tex
\documentclass[twoside]{article}

\usepackage[utf8]{inputenc} 
\usepackage[T1]{fontenc}    
\usepackage{hyperref}       
\usepackage{url}            
\usepackage{booktabs}       
\usepackage{amsfonts}       
\usepackage{nicefrac}       
\usepackage{microtype}      
\usepackage{amsmath}
\usepackage{algorithm}
\usepackage{algorithmic}
\usepackage{amsthm}
\usepackage{bm}
\usepackage{xcolor}
\newtheorem{definition}{Definition}

\include{mlVecMat}

\usepackage[round]{natbib}

\usepackage{enumitem}
\usepackage{breakcites}

\usepackage{titlesec}
\titlespacing{\paragraph}{%
  0pt}{
  0.0\baselineskip}{
  1em}

\usepackage[accepted]{aistats2020}
\begin{document}

%

%

\twocolumn[

\aistatstitle{AMAGOLD: Amortized Metropolis Adjustment for Efficient Stochastic Gradient MCMC}

\aistatsauthor{ Ruqi Zhang \And A. Feder Cooper \And Christopher De Sa }

\aistatsaddress{ Cornell University \And  Cornell University \And Cornell University } ]
\color{black}
\begin{abstract}
Stochastic gradient Hamiltonian Monte Carlo (SGHMC) is an efficient method for sampling from continuous distributions. It is a faster alternative to HMC: instead of using the whole dataset at each iteration, SGHMC uses only a subsample. This improves performance, but introduces bias that can cause SGHMC to converge to the wrong distribution. One can prevent this using a step size that decays to zero, but such a step size schedule can drastically slow down convergence. To address this tension, we propose a novel second-order SG-MCMC algorithm\textemdash AMAGOLD\textemdash that \textit{infrequently} uses Metropolis-Hastings (M-H) corrections to remove bias. The infrequency of corrections amortizes their cost. We prove AMAGOLD converges to the target distribution with a fixed, rather than a diminishing, step size, and that its convergence rate is at most a constant factor slower than a full-batch baseline. We empirically demonstrate AMAGOLD's effectiveness on synthetic distributions, Bayesian logistic regression, and Bayesian neural networks.
\end{abstract}

\section{Introduction}\label{sec:intro}
Markov chain Monte Carlo (MCMC) methods play an important role in Bayesian inference. They work by constructing a Markov chain with the desired distribution as its equilibrium distribution; one samples from the chain and, as the algorithm converges to its equilibrium, the samples drawn reflect the desired distribution \citep{metropolis1953equation,duane1987hybrid,horowitz1991generalized,neal2011mcmc}. Although MCMC is a powerful technique, when the sampled distribution depends on a very large dataset,  its performance is often limited by the dataset's size---usually by the cost of computing sums over the entire dataset.

One approach for scaling MCMC is to decouple the algorithm from the size of the dataset, using stochastic gradients in lieu of full-batch gradients \citep{welling2011bayesian, chen2014stochastic, ding2014bayesian}. This family of methods is called Stochastic gradient MCMC (SG-MCMC). For example, stochastic gradient Langevin dynamics (SGLD) replaces the gradient with a stochastic estimate in first order Langevin Monte Carlo (LMC) \citep{welling2011bayesian}. The second-order analog of SGLD is stochastic gradient Hamiltonian Monte Carlo (SGHMC). SGHMC can be thought of as a stochastic version of the popular Hamiltonian Monte Carlo (HMC) algorithm and second-order Langevin dynamics (L2MC) \citep{chen2014stochastic}. These algorithms have been particularly useful in Bayesian neural networks, and many variants have been proposed to increase their sampling efficiency and accuracy \citep{ding2014bayesian,ahn2012bayesian,ma2015complete,zhang2017stochastic}. Table \ref{tab:terms} provides a clarifying summary of the algorithms considered in this paper.

The runtime efficiency benefits of such SG-MCMC methods also come with a drawback: stochastic gradients introduce bias. Bias can cause convergence to a stationary distribution that differs from the one we wanted to sample from, and usually comes from two sources: converting the continuous-time process into discrete gradient updates, and noise from stochastic gradient estimates. These sources of error are in a sense unavoidable because they are also the source of SG-MCMC's runtime efficiency. Discretization with a large step size $\epsilon$ (instead of diminishing $\epsilon \rightarrow 0$) allows SG-MCMC to quickly move around its state space, and using stochastic gradients is key for scalability.

The standard approach for removing bias from a Markov chain is to introduce a \emph{Metropolis-Hastings (M-H) correction} \citep{metropolis1953equation}. This involves rejecting some fraction of the chain's transitions to restore the correct stationary distribution. Na\"ively applying M-H to SG-MCMC algorithms is often computationally prohibitive because the M-H step typically needs to sum over the entire dataset. Performing this expensive computation every iteration would defeat the purpose of using stochastic gradients to improve performance. Thus, more sophisticated techniques are needed to achieve efficient, unbiased sampling for SG-MCMC.

\begin{table}[t]
	\centering
	\begin{tabular}{@{}lcc@{}} \toprule
		Algorithm & Exact? & Stochastic Gradient? \\ \midrule
		AMAGOLD & Yes & Yes   \\
		L2MC & Yes & No \\
		HMC & Yes & No \\
		SGHMC & No & Yes  \\ 
		\bottomrule
	\end{tabular}
	\caption{Comparing 2nd order MCMC methods.}
	\label{tab:terms}
\end{table}

In this paper, we show that asymptotic exactness is possible \textit{without being prohibitively expensive}. Specifically, we propose \emph{Amortized Metropolis-Adjusted stochastic Gradient second-Order Langevin Dynamics (AMAGOLD)}. It achieves asymptotic exactness for SGHMC by using an M-H correction step and does so without obliterating the performance gains provided by stochasticity. Our key insight is to apply the M-H step \textit{infrequently}. Rather than computing it every update, AMAGOLD performs it every $T$ steps ($T>0$). We prove this is sufficient to remove bias while also improving performance by amortizing the M-H correction cost over $T$ steps. We develop both reversible and non-reversible AMAGOLD variants and prove both converge to the desired distribution. We also prove a convergence rate relative to using full-batch gradients, which cleanly captures the effect of using stochastic gradients. This result provides insight about the trade-off between minibatching speed-ups and the convergence rate. Our results also show the noise from stochastic gradients has a provably bounded effect on convergence. In summary, our contributions are as follows:
\begin{itemize}[itemsep=0pt,topsep=0pt]
    \item We introduce AMAGOLD, an efficient, asymptotically-exact SGHMC variant that \textit{infrequently} applies an M-H correction. We give reversible and non-reversible versions.
    \item We guarantee AMAGOLD converges to the target distribution, and does so without requiring step size $\epsilon\rightarrow 0$ or precise noise variance estimation.
    \item We prove a bound on AMAGOLD's convergence rate with mild assumptions, measured by the spectral gap. This bound is relative to how fast the algorithm \emph{would have} converged if full-batch gradients were used. This is the first such relative convergence bound we are aware of for SG-MCMC.
    \item We validate our convergence guarantees empirically.  Comparing to SGHMC, AMAGOLD is more robust to hyperparameters. Regarding performance, AMAGOLD is competitive with full-batch baselines on synthetic and real-world datasets, and outperforms  SGHMC on various tasks.
\end{itemize}

\begin{algorithm}[t]
\caption{SGHMC}
\label{alg:SGHMC}
\begin{algorithmic}[1]
  \STATE \textbf{given:} Energy $U$, initial state $\theta \in \Theta$, step size $\epsilon$, momentum variance $\sigma^2$, friction $\beta$
  \LOOP
    \STATE \textbf{optionally, resample momentum: } 
    \STATE $r \sim \mathcal{N}(0, \sigma^2)$
    \STATE \textbf{initialize position and momentum: } 
    \STATE$r_{\frac{1}{2}} \leftarrow r $, $\theta_0 \leftarrow \theta$
    \FOR {$t = 1$ \textbf{to} $T$}
    \STATE \textbf{position update:} 
    $\theta_{t} \leftarrow \theta_{t-1} + \epsilon \sigma^{-2} r_{t-\frac{1}{2}}$
    \STATE \textbf{sample noise } $\eta_t \sim \mathcal{N}(0, 4 \epsilon \beta \sigma^2)$
    \STATE \textbf{sample random energy component} $\tilde U_t$
    \STATE \textbf{update momentum: } \[r_{t + \frac{1}{2}} \leftarrow r_{t - \frac{1}{2}} - \epsilon \nabla \tilde U_t(\theta_t) - 2\epsilon \beta r_{t - \frac{1}{2}}  + \eta_t\]
    \ENDFOR
    \STATE \textbf{new values: }$(\theta, r) \leftarrow (\theta_{T}, r_{T+\frac{1}{2}})$ 
    \STATE $\triangleright$ no M-H step
  \ENDLOOP
\end{algorithmic}
\end{algorithm}

\section{Related Work}

Our work is situated within a rich literature of SG-MCMC variants that take advantage of stochastic gradient techniques. These methods have demonstrated success on deep neural networks (DNNs) for various tasks \citep{li2016learning,gan2016scalable,zhang2019cyclical}. In particular, second-order SG-MCMC methods like SGHMC, which have a momentum term, have been shown to outperform first-order methods like SGLD on many applications \citep{chen2014stochastic,chen2015convergence}. \citet{gao2018global} proves SGHMC's convergence can be faster than SGLD's on non-convex problem due to its momentum-based acceleration. SGHMC can also be thought of as a stochastic version of L2MC \citep{horowitz1991generalized} or HMC; we therefore use both L2MC and HMC as experimental full-batch baselines.

Prior work has also studied SGHMC's convergence properties. \citet{chen2014stochastic} examines its convergence for ``asymptotically'' small step sizes, in which a continuous-time system governs the dynamics (in contrast, our algorithm is asymptotically exact with a constant step size). Other work proves convergence with high-order integrators \citep{chen2015convergence} and obtains non-asymptotic convergence bounds for SGHMC on non-convex optimization tasks \citep{gao2018global}.

Additional work has studied the properties of first-order M-H adjusted Langevin methods, such as MALA \citep{grenander1994representations,roberts1996exponential,roberts1998optimal,roberts2002langevin,stramer1999langevin}. \citet{dwivedi2018log} derives the mixing time of MALA for strongly log-concave densities, showing it has a better convergence rate than unadjusted Langevin (in comparison, AMAGOLD does not require the assumption of strongly log-concave densities). \citet{korattikara2014austerity} developed a minibatch M-H approach, which uses subsampling in the M-H correction step, and applied it to correct bias in SGLD. They show cases where SGLD diverges from the target distribution, while SGLD with a minibatched M-H correction performs well.

The work above involves first-order methods. To the best of our knowledge, we are the first to develop an unbiased, efficient second-order SG-MCMC algorithm. We are also the first in this space to use the spectral gap, a traditional metric to evaluate MCMC convergence \citep{hairer2014spectral,levin2017markov,de2018minibatch}. It requires milder assumptions than techniques in prior SG-MCMC work, such as 2-Wasserstein \citep{raginsky2017non,dalalyan2019user}, mean squared error \citep{vollmer2016exploration,chen2015convergence}, and empirical risk \citep{gao2018global}. 

\section{Preliminaries}
We start by briefly describing the standard setup of Bayesian inference. Given some dataset $\mathcal{D}$ and domain $\Theta$, suppose we are interested in sampling from the posterior distribution 
$\pi(\theta) \propto \exp\left(- U(\theta)\right)$
where
\vspace{-.1cm}
\[
U(\theta) = -\sum_{x\in\mathcal{D}} \log p(x|\theta) - \log p(\theta).
\]
$U(\theta)$ is the \emph{energy function},
$\theta$ ranges over $\Theta$, and $\pi \propto \mu$ denotes $\pi$ is the unique distribution with PDF proportional to $\mu$. One way to compute this distribution is to construct a Markov chain with stationary distribution $\pi$ and run it to produce a sequence of samples. 

A second-order chain, such as HMC, SGHMC, or L2MC \citep{duane1987hybrid,horowitz1991generalized,neal2011mcmc}, does this by \emph{augmenting} the state space with an additional momentum variable $r$, giving joint distribution
\vspace{-.25cm}
\begin{align*}
  \pi(\theta, r)
  \propto
  \exp(-H(\theta, r))
  =
  \exp\left(-U(\theta) - \frac{1}{2\sigma^2} \| r \|^2 \right),
\end{align*}
where $H$ is the \emph{Hamiltonian}, which measures the total energy of the system. Note we could replace the norm with any positive definite quadratic form on $r$. For simplicity, we only consider the case of isotropic momentum energy\textemdash where the \emph{mass matrix} is $\sigma^2 I$.
HMC then simulates Hamiltonian dynamics
\begin{equation}
    \label{eqnHamDyn}
     d\theta = \sigma^{-2} r \; dt, \hspace{2em}
     dr = - \nabla U(\theta) \; dt. 
\end{equation}
The value of the Hamiltonian is preserved under these dynamics, so we must also include transitions that change the value of $H$ to explore the whole state space. HMC does this by periodically resampling $r$ from its conditional distribution. L2MC does so by continuously modifying $r$ with a friction term and added Gaussian noise. Even though (\ref{eqnHamDyn}) preserves $H$, the discrete simulation of (\ref{eqnHamDyn}) run by HMC or L2MC \textit{does not} necessarily do so. Therefore both algorithms need an M-H correction step to prevent bias due to discretization.

SGHMC (Algorithm~\ref{alg:SGHMC}) reduces the computational cost of these methods by using a stochastic gradient in lieu of the full-batch gradient $\nabla U$. It estimates $U(\theta)$ using minibatch $\tilde{ \mathcal{D}}$:
\[
\tilde U(\theta) \approx -\frac{\Abs{\mathcal{D}}}{|\tilde{ \mathcal{D}}|}\sum_{x\in\tilde{\mathcal{D}}}\log p(x|\theta) - \log p(\theta).
\]
However, using a minibatch introduces noise; na\"ively replacing $U$ by $\tilde U$ leads to divergence from the target distribution. To offset this noise, SGHMC adds the friction term from L2MC (Appendix~\ref{sec:append:sghmc}). SGHMC uses the leapfrog algorithm to discretize the system \citep{neal2011mcmc}. Notably, SGHMC does not include an M-H correction; to reduce the bias, it requires small $\epsilon$. 

\section{Amortized Metropolis Adjustment}
\label{sec:AMA}
\emph{Reversible} Markov chains are a particularly well-studied and well-behaved class of Markov chains.
A Markov chain with transition probability operator $G$ is reversible (also called satisfying the \emph{detailed balance condition}) if
for any pair of states $x$ and $y$
\begin{align}
\label{eq:detail-balance}
\pi(x) G(x, y) = \pi(y) G(y, x).
\end{align}
It is well-known that a chain satisfying (\ref{eq:detail-balance}) has stationary distribution $\pi$.
An M-H correction constructs a reversible chain $G$ with stationary distribution $\pi$ from any Markov chain $P$ (called the \emph{proposal distribution}) by doing the following at each iteration.
First, starting from state $x$, sample $y$ from the proposal distribution $P(x,y)$.
Second, compute the \emph{acceptance probability}
\[
\tau = \min\left(1, \frac{\pi(y)P(y,x)}{\pi(x)P(x,y)}\right).
\]
Finally, with probability $\tau$ transition to state $y$; otherwise, remain in state $x$.
This correction results in a reversible chain with stationary distribution $\pi$; however, computing $\tau$ at every step can be costly. 

The natural way to amortize the cost of running M-H is to replace the single proposal of baseline M-H with $T$ proposal-chain steps. This divides its cost among $T$ iterations of the underlying chain, effectively decreasing it by a factor of $T$. 
For stochastic MCMC, each proposal step can be written as $P(x,y;\zeta)$, which denotes the probability of going from state $x$ to state $y$ given stochastic sample $\zeta$ taken from some known distribution. (In minibatched MCMC, $\zeta$ captures information about which data we sample at that step.)
Using this, we can run the following algorithm, starting at $x$.
First, set $x_0 = x$, and run for $t$ from $0$ to $T - 1$
\[
    \text{sample noise } \zeta_t \text{, then sample } x_{t+1} \sim P(x_t,x_{t+1};\zeta_t).
\]
Next, set $y = x_T$: this is the proposal run an M-H correction on.
Finally, compute the acceptance probability 
\begin{align}
\tau
&=
\min\left(1, \frac{\pi(y)}{\pi(x)} \prod_{t = 0}^{T-1} \frac{P(x_{t+1},x_t; \zeta_t)}{P(x_t,x_{t+1}; \zeta_t)}\right) \label{eqn:amortmcmc} \\
&=
\min\left(1, \prod_{t = 0}^{T-1} \frac{\pi(x_{t+1})P(x_{t+1},x_t; \zeta_t)}{\pi(x_t)P(x_t,x_{t+1}; \zeta_t)}\right). \label{eqn:amortmcmc2}
\end{align}
and transition to state $y$ with probability $\tau$; otherwise, remain in state $x$.

It is straightforward to see this algorithm results in a reversible chain with stationary distribution $\pi$.\footnote{\label{noteproof}A detailed proof appears in Appendix~\ref{sec:append:AMA}.}
Additionally, this amortized M-H step (\ref{eqn:amortmcmc}) is easily computed as long as the probabilities $P(\cdot,\cdot;\zeta)$ are tractable.

We expect this approach will be effective when the M-H step does ``not reject too often''. This will certainly be the case when the terms inside the product in (\ref{eqn:amortmcmc2}) all tend to be close to $1$, which happens when the proposals $P(x, y; \eta)$ are ``close'' to being reversible with stationary distribution $\pi$. This is a good heuristic: our amortization approach should be effective when the proposals are close to being reversible.

Unfortunately, this heuristic does not apply to SGHMC's proposal step since SGHMC and other Hamiltonian-like steps are not close to satisfying the reversibility condition (\ref{eq:detail-balance}).
Instead, the natural ``reverse'' trajectory for a Hamiltonian step reverses the order of the states and \emph{negates the momentum}. 
The analog of reversibility for this sort of step is \emph{skew-reversibility}~\citep{turitsyn2011irreversible}.
Given some measure-preserving involution over the state space denoted $x^\bot$, a chain $G$ is skew-reversible if $\pi(x) = \pi(x^{\bot})$ and
\begin{equation}
\label{eq:skew-detail-balance}
\pi(x) G(x, y) = \pi(y^{\bot}) G(y^{\bot}, x^{\bot}).
\end{equation}
Concretely, for Hamiltonian dynamics we use the involution that negates the momentum, i.e. $(\theta, r)^{\bot} = (\theta, -r)$. It is straightforward to show that a skew-reversible Markov chain also has $\pi$ as its stationary distribution.\textsuperscript{\ref{noteproof}}
Such non-reversible chains have attracted a great deal of recent attention because they are more efficient than reversible ones in some situations~\citep{turitsyn2011irreversible,hukushima2013irreversible,ma2016unifying}.

A natural consequence of this setup is that we can amortize M-H in the same manner as before, using skew-reversibility in place of reversibility.
This gives the same multi-step-proposal algorithm as before, except that the acceptance probability is replaced with 
\begin{align}
    \tau
    &=
    \min\left(1, \frac{\pi(y^{\bot})}{\pi(x)} \prod_{t = 0}^{T-1} \frac{P(x_{t+1}^{\bot},x_t^{\bot}; \zeta_t)}{P(x_t,x_{t+1}; \zeta_t)}\right) \label{eqn:amortmcmcskew}.
\end{align}
The resulting corrected chain will be skew-reversible with stationary distribution $\pi$.\textsuperscript{\ref{noteproof}}
Intuitively, this chain will ``not reject too often'' as long as the proposals $P$ are ``close'' to being skew-reversible.
Since SGHMC steps are close to being skew-reversible, this is the more natural approach for amortizing M-H, rather than using (\ref{eqn:amortmcmc}). 
If one wants to use the well-developed theoretical tools for a reversible chain, it is known that we can recover a reversible chain from a skew-reversible one by simply resampling the momentum at the beginning of the outer loop.\textsuperscript{\ref{noteproof}} Note this reversible chain can be different from the one obtained by using condition (\ref{eqn:amortmcmc}).

\section{AMAGOLD}\label{sec:AMAGOLD}
We now apply the amortized Metropolis adjustment (AMA) method of Section~\ref{sec:AMA} to second-order SG-MCMC. 
As a proposal, we use the stochastic leapfrog step that starts in $(\theta, r)$ and proposes $(\theta^*, r^*)$ by running
{\small\begin{align*}
    \theta_0 &= \theta + \frac{1}{2} \epsilon \sigma^{-2} r \\
    r^* &= ( 
        (1 - \epsilon \beta) r - \epsilon \nabla \tilde U_t(\theta_0) + \mathcal{N}(0, 4 \epsilon \beta \sigma^2 I)
    )/(1 + \epsilon \beta) \\
    \theta^* &= \theta_0 + \frac{1}{2} \epsilon \sigma^{-2} r^*.
\end{align*}}%
Applying our amortized M-H correction to this proposal step using the acceptance probability (\ref{eqn:amortmcmcskew}) results in AMAGOLD (Algorithm~\ref{alg:IMA}).
AMAGOLD is, by construction, skew-reversible, and we have the option of making it reversible by resampling the momentum. 

AMAGOLD has three key differences compared to SGHMC. First, motivated by the time-reversal-symmetric nature of conditions (\ref{eq:detail-balance}) and (\ref{eq:skew-detail-balance}), we use a clearly time-symmetric update step in the inner loop (compare Line 12 of Algorithm~\ref{alg:IMA}, which can be written as $r_{t + \frac{1}{2}} \leftarrow r_{t - \frac{1}{2}} - \epsilon \nabla \tilde U_t(\theta_t) - \epsilon\beta(r_{t - \frac{1}{2}}+r_{t + \frac{1}{2}}) + \eta_t$, with the less clearly symmetric Line 11 of Algorithm~\ref{alg:SGHMC}). Note this is just a different way of writing the algorithm: the update steps could be made equivalent by appropriately setting the hyperparameters. Second, we use a type of leapfrog integration that starts and ends the outer loop with a half-position-update (Lines 5 and 15). This too is done in the interest of time-reversal-symmetry. Third, there is an additional term $\rho$ in AMAGOLD, which we call the \emph{energy accumulator}, which accumulates the $\log$ of the product in (\ref{eqn:amortmcmcskew}). Computing $\rho$ requires little extra cost since all its terms are already obtained in the standard update. AMAGOLD is thus unbiased without adding too much cost over SGHMC. The following theorem summarizes AMAGOLD's asymptotic accuracy. (This follows from the construction; an explicit proof is in Appendix~\ref{sec:thm1}.)
\begin{theorem}\label{thm:convergence}
Consider the Markov chain described by AMAGOLD (Algorithm~\ref{alg:IMA}).
If the momentum is resampled (on line 3), then this Markov chain is reversible.
Otherwise the Markov chain is skew-reversible. In either case, its stationary distribution is $\pi$.
\end{theorem}

\begin{algorithm}[t]
\caption{AMAGOLD}
\label{alg:IMA}
\begin{algorithmic}[1]
  \STATE \textbf{given:} Energy $U$, initial state $\theta \in \Theta$,  step size $\epsilon$, momentum variance $\sigma^2$, friction $\beta$
  \vspace{0.5em}
  \LOOP
    \STATE \textbf{optionally, resample momentum:} 
    
    $r \sim \mathcal{N}(0, \sigma^2 I)$
    \STATE \textbf{initialize momentum, energy acc: } 
    
    $r_{-\frac{1}{2}} \leftarrow r$, $\rho_{- \frac{1}{2}} \leftarrow 0$
    \STATE \textbf{half position update:} $\theta_{0} \leftarrow \theta + \frac{1}{2} \epsilon \sigma^{-2} r_{-\frac{1}{2}}$
    \vspace{0.5em}
    \FOR {$t = 0$ \textbf{to} $T-1$}
      \IF {$t\neq 0$ }
        \STATE \textbf{position update:} $\theta_{t} \leftarrow \theta_{t-1} + \epsilon \sigma^{-2} r_{t-\frac{1}{2}}$
    \ENDIF
    \vspace{0.5em}
    \STATE \textbf{sample noise } $\eta_t \sim \mathcal{N}(0, 4 \epsilon \beta \sigma^2 I)$
    \STATE \textbf{sample random energy component} $\tilde U_t$
    \vspace{0.5em}
    \STATE \textbf{update momentum: } 
    
    $r_{t + \frac{1}{2}} \leftarrow ((1-\epsilon \beta)r_{t - \frac{1}{2}} - \epsilon \nabla \tilde U_t(\theta_t) + \eta_t)/(1+\epsilon \beta)$
    \STATE \textbf{update energy acc: }
    
    \hspace{-.07em}$\rho_{t + \frac{1}{2}} \leftarrow \rho_{t - \frac{1}{2}} + \frac{1}{2} \epsilon \sigma^{-2}  \nabla \tilde U_t(\theta_t)^T \left(r_{t - \frac{1}{2}} + r_{t + \frac{1}{2}}\right)$\vspace{0.5em}
    \ENDFOR
    \vspace{0.5em}
    \STATE \textbf{half position update:} 
    
    $\theta_T \leftarrow \theta_{T-1} + \frac{1}{2} \epsilon \sigma^{-2} r_{T-\frac{1}{2}}$\vspace{0.5em}
    \STATE \textbf{new values:} $\theta^* \leftarrow \theta_T$, $r^* \leftarrow r_{T-\frac{1}{2}}$\vspace{0.25em}
    \STATE $a \leftarrow \exp\left( U(\theta) - U(\theta^*) + \rho_{T - \frac{1}{2}} \right)$\vspace{0.25em}
    \STATE \textbf{with probability } $\min(1, a)$,
    
    \hspace{1em}\textbf{ update } $\theta \leftarrow \theta^*$, $r \leftarrow r^*$ (as long as $\theta^* \in \Theta$)
    \STATE \textbf{otherwise update } $r \leftarrow - r_{-\frac{1}{2}}$
  \ENDLOOP
\end{algorithmic}
\end{algorithm}
\vspace{-1em}
\paragraph{Connection to previous methods}
AMAGOLD is related to several previous MCMC methods. When using a full-batch gradient, AMAGOLD becomes L2MC with amortized M-H-adjustment. 
Using a full-batch, $\beta=0$, and resampling, AMAGOLD becomes HMC (Appendix \ref{sec:relation-hmc}). 
If we disable AMAGOLD's M-H step (and adjust hyperparameters), it becomes SGHMC.

\paragraph{Illustrating AMAGOLD}

\begin{figure}[t!]
    \vspace{-0mm}
    \begin{tabular}{c c}
        \centering
        \hspace{-.25cm}
        \includegraphics[width=4cm]{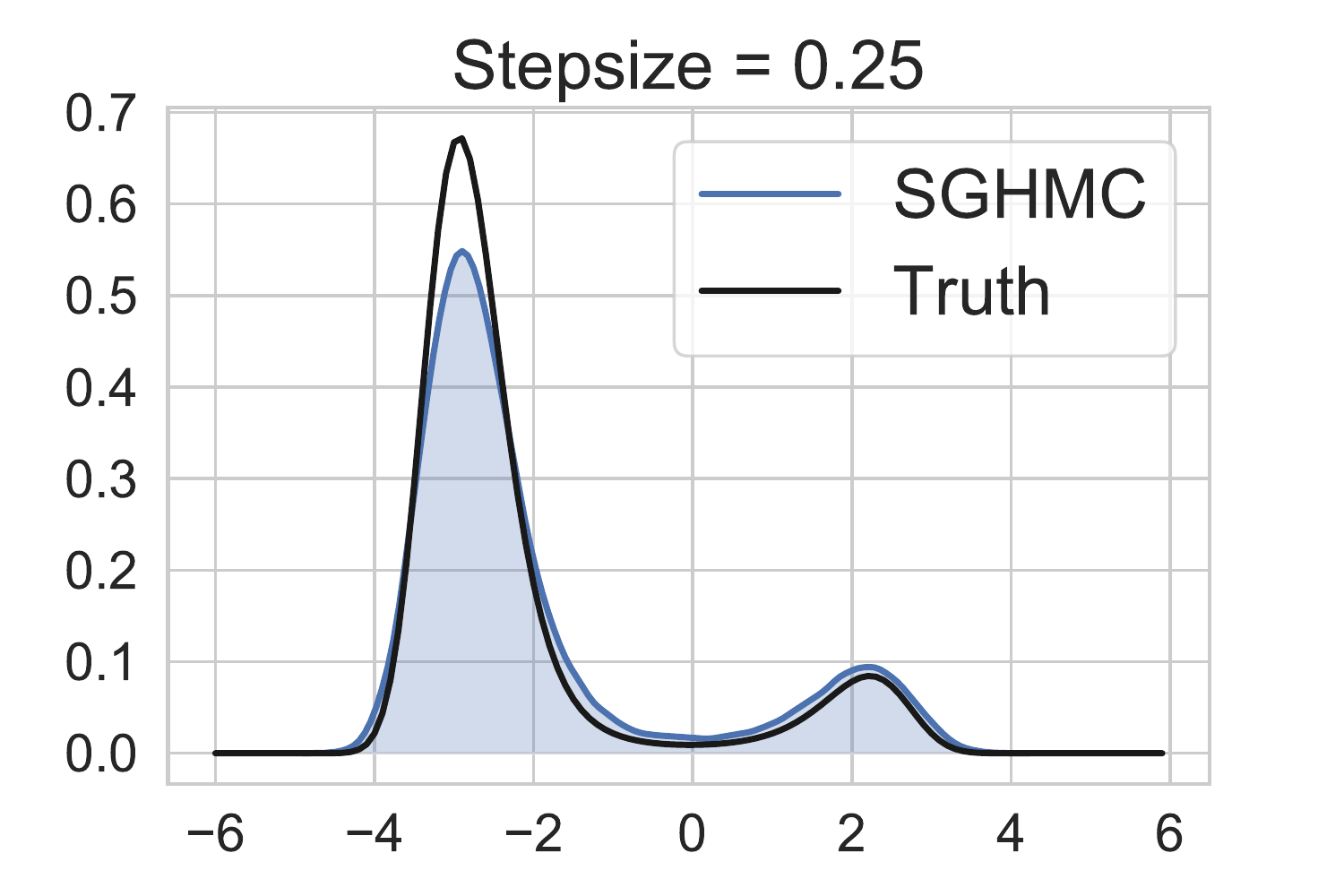} &
        \hspace{-3mm}
        \includegraphics[width=4cm]{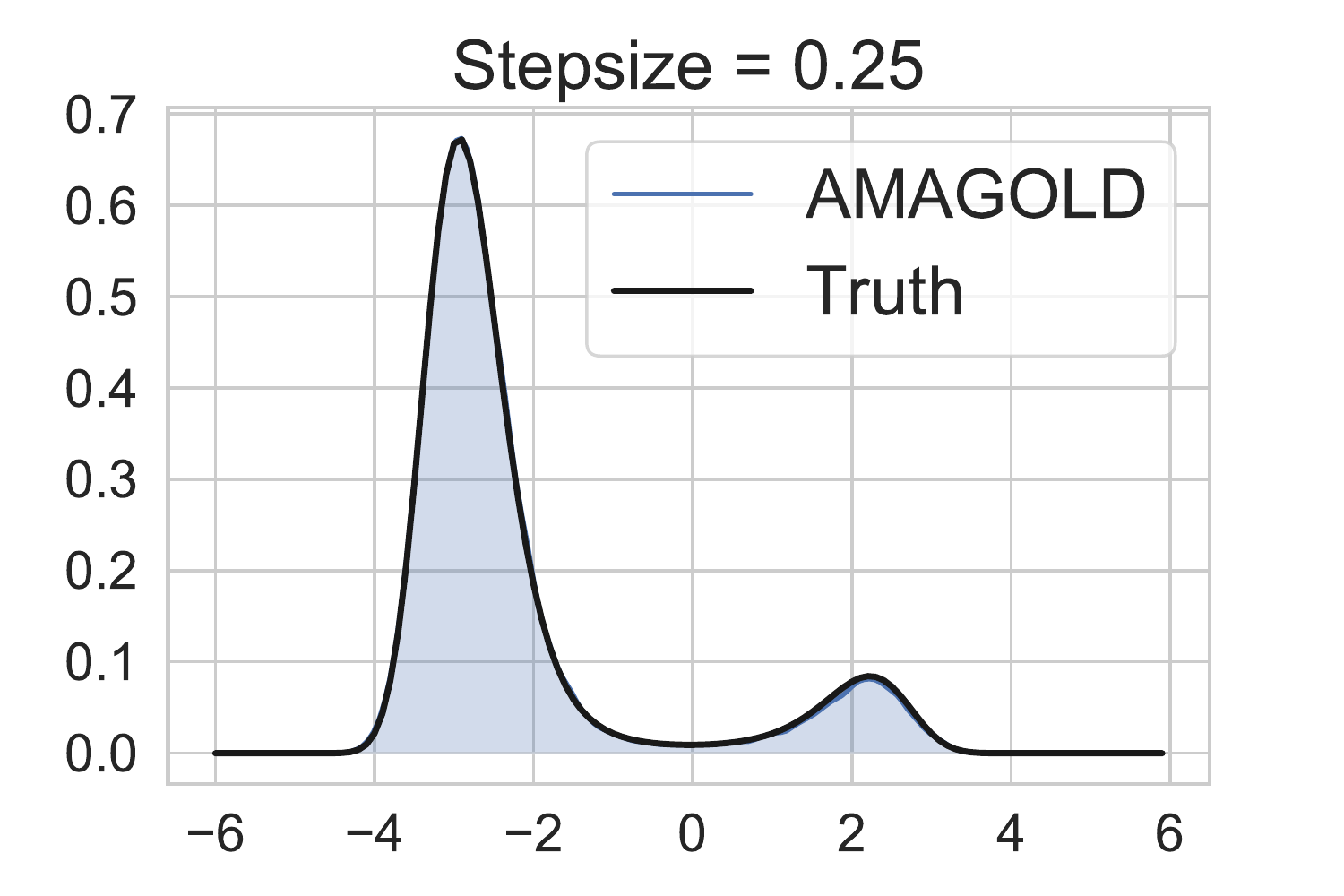}  \\
        (a) SGHMC \vspace{-0mm} & \hspace{-0mm}
        (b) AMAGOLD\hspace{-0mm} \hspace{0mm} \\
        \hspace{-.5cm}
        \includegraphics[width=4cm]{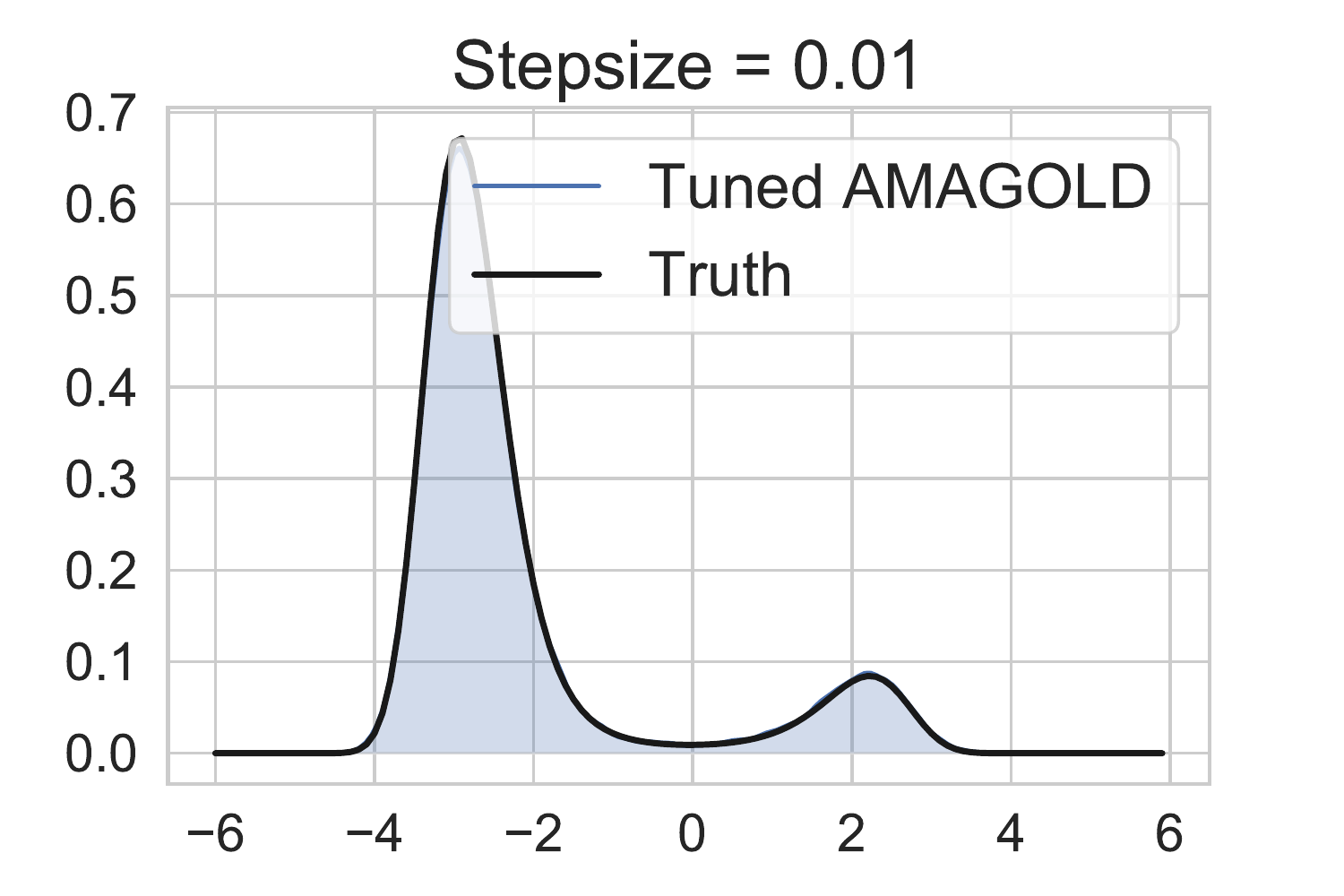} &
        \hspace{-2mm}
        \includegraphics[width=3.9cm]{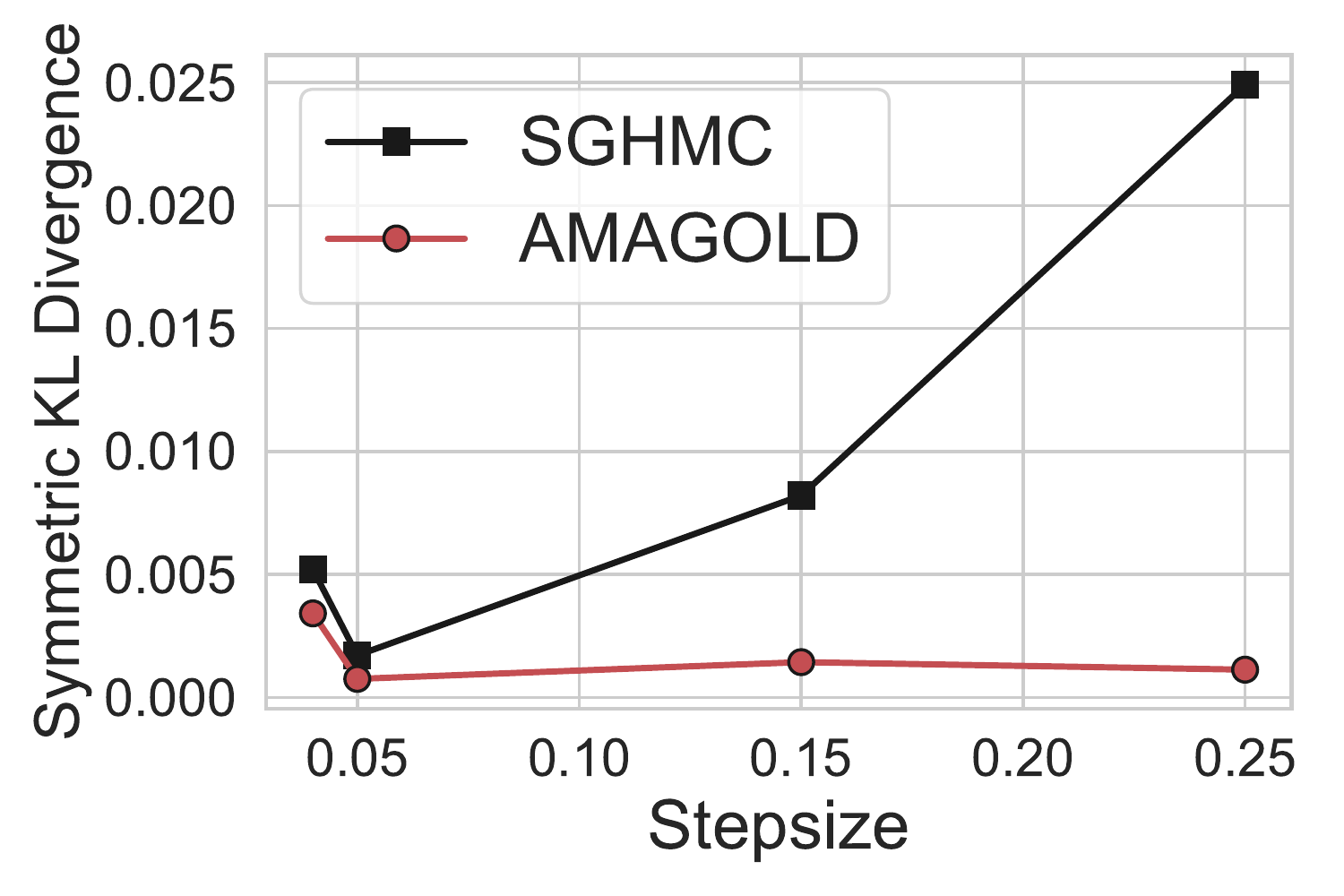}
        \\
        (c) Tuned AMAGOLD\hspace{-0mm} &
        (d)  KL Divergence \vspace{-0mm}
    \end{tabular}
    \vspace{-0mm}
    \caption{Estimated densities of (a) SGHMC and (b) AMAGOLD for step size 0.25 compared to the ground truth and (c) step size 0.01 for tuned AMAGOLD (see Section \ref{sec:tuning}); (d) Comparison of symmetric KL divergence, varying step sizes for SGHMC and AMAGOLD.}
    \vspace{-3mm}
    \label{fig:1dim}
\end{figure}

To illustrate AMAGOLD is able to achieve unbiased stochastic MCMC, we test our method on a double-well potential \citep{ding2014bayesian, li2016high}:
\[
U(\theta) = (\theta + 4)(\theta + 1)(\theta - 1)(\theta - 3)/14 + 0.5.
\]
The target distribution is proportional to $\exp(-U(\theta))$. To simulate stochastic gradients, we let $\nabla \tilde U = \nabla U + \mathcal{N}(0, 1)$. We show the results of SGHMC and AMAGOLD when $\beta = 0.25$, $T=10$ and $\epsilon = 0.25$. Results for $\epsilon = \{0.05, 0.15\}$ are in Appendix~\ref{sec:append:dwp}. 

In Figure~\ref{fig:1dim} and Appendix~\ref{sec:append:dwp}, the estimated densities of AMAGOLD are very close to the true density on varying step sizes. In contrast, SGHMC does not converge to the correct distribution asymptotically. This is especially the case when the step size is large: SGHMC diverges from the true distribution. These observations validate Theorem~\ref{thm:convergence}, as AMA guarantees convergence to the target distribution. To quantitatively measure divergence from the true distribution, we plot the symmetric KL divergence as a function of the step size in Figure~\ref{fig:1dim}d. We can see that SGHMC is very sensitive to step size, and may require careful tuning in practice, while AMAGOLD is more robust.

\subsection{Convergence Rate Analysis}

Using stochastic gradients in MCMC can reduce the cost of each iteration. However, this does not mean the overall cost of the algorithm will be less in comparison to its non-stochastic counterpart. Rather, it is possible that the stochastic chain's convergence rate becomes much slower than the non-stochastic one. 
To be confident in the effectiveness of an SG-MCMC method, we must rule this out: We must show that the convergence speed of the stochastic chain is not slowed down, or at least not too much, compared to the non-stochastic chain. We do this analysis for AMAGOLD as follows.

Since AMAGOLD can be regarded as stochastic L2MC, we study reversible AMAGOLD's convergence rate relative to L2MC with an amortized M-H correction. Prior work has used this type of bound to prove the convergence rate of subsampled MCMC methods \citep{de2018minibatch,zhang2019poisson}. Unlike work that uses 2-Wasserstein, MSE, or empirical risk minimization to evaluate the convergence rate of SG-MCMC, we are the first to use the spectral gap---a traditional metric for evaluating MCMC convergence \citep{hairer2014spectral,levin2017markov} that is directly related to another common measurement, the mixing time \citep{levin2017markov}. Our bound only requires mild assumptions compared to prior work \citep{vollmer2016exploration,chen2015convergence}, and we measure convergence to the target distribution directly, rather than empirical risk minimization \citep{gao2018global}.

The spectral gap $\gamma$ of a reversible Markov chain with transition probability operator $G$ is defined as the smallest distance between any non-principal eigenvalue of $G$ and $1$, the principal eigenvalue of $G$ \citep{levin2017markov}.
The spectral gap determines the convergence rate of a Markov chain: a chain with a smaller $\gamma$ will take longer to converge. To ensure the existence of $\gamma$, we assume geometric ergodicity of the full-batch chain \citep{rudolf2011explicit}.
To bound $\gamma$, we assume the covariance of the gradient samples of AMAGOLD is bounded isotropically with
\vspace{-.1cm}
\begin{align*}
    \Exv{(\nabla \tilde U(\theta) - \nabla U(\theta))(\nabla \tilde U(\theta) - \nabla U(\theta))^T}
    \preceq
    \frac{V^2}{d} I
\end{align*}
for some constant $V > 0$.
This sort of bounded-variance assumption is standard in the analysis of stochastic gradient algorithms.

The following theorem shows that with appropriate hyperparameter settings the convergence rate of AMAGOLD will not be slowed down by more than a constant factor.
\begin{theorem}\label{thm:convergence-rate}
For some parameters $\epsilon > 0$, $\sigma > 0$, and $\beta > 0$, let $\bar \gamma$ denote the spectral gap of the L2MC chain running with parameters $(\epsilon, \sigma, \beta)$.
Assume that these parameters are such that $\epsilon V^2 \le 4 \sigma^2 \beta d$.
Define a constant
$c
=
1
+
\sqrt{
  \frac{\epsilon V^2}{16 \sigma^2 \beta T d^2}
}$.
Let $\gamma$ denote the spectral gap of AMAGOLD running with parameters $(\epsilon, \sigma \cdot c^{-1/4}, \beta \cdot c^{-1/2})$.
Then,
\[
\frac{\gamma}{\bar \gamma}
\ge
\exp\left( -\frac{\epsilon T V^2}{4 \sigma^2 \beta} - 
   \sqrt{
      \frac{\epsilon T V^2}{\sigma^2 \beta}
    }
  \right).
\]
\end{theorem}
This requirement on parameters is easy to satisfy because $d$ is generally large and $\epsilon$ is generally small in practice. For the same reason, $c$ is usually close to 1, so the parameters used by the two chains are very close.

This theorem has three useful takeaways:
First, AMAGOLD's convergence rate is essentially the same as L2MC up to a constant, which will approach $1$ as the batch size increases ($V$ decreases) or $\epsilon$ decreases. 
Second, it shows the effect of minibatching on convergence rate: if one reduces the minibatch size (i.e. $V^2$ increases), they can expect the convergence rate to decrease with a rate of $\exp(-O(V^2))$. 
Third, the theorem outlines a range of parameters (where $\epsilon T V^2 \ll \sigma^2 \beta$) 
over which AMAGOLD converges at a similar rate to the full-batch algorithm.

\subsection{AMAGOLD in Practice}\label{sec:practice}
Here we describe some simple modifications that can further improve AMAGOLD's performance.

\paragraph{Minibatch M-H}
Amortizing the cost of an M-H correction over $T$ steps is not always sufficient for achieving good performance on large datasets. This is because calculating the true energy $U$ requires a scan over the whole dataset. We can further reduce the cost of a single correction by using minibatch M-H to compute the acceptance probability\textemdash using a minibatch at line 17 of Algorithm~\ref{alg:IMA}. Prior work has estimated the M-H correction using a subset of the data \citep{korattikara2014austerity, bardenet2014towards, maclaurin2015firefly, seita2016efficient}. These methods are composable with, rather than exclusive with, AMAGOLD and could provide additional speed-ups.

\paragraph{Tuning the step size} \label{sec:tuning}
Our experiments on double well potential (Figure~\ref{fig:1dim}) show step size significantly influences SGHMC's performance. Besides being more robust to step size, AMAGOLD's step size can be more easily tuned. The M-H step's acceptance probability provides information about whether a step size is desirable. Based on this information, the step size can be tuned automatically to target some fixed acceptance probability during burn-in without affecting convergence. With a fixed step size $\epsilon = 0.01$, both AMAGOLD and SGHMC provide poor density estimates due to too small step size. However, when we let AMAGOLD adjust $\epsilon$ such that the average M-H acceptance probability is $85\%$, it estimates the density accurately (Figure~\ref{fig:1dim}c, Appendix~\ref{sec:append:dwp}).

\section{Experiments}\label{sec:experiments}
Here we validate our theory empirically and explore the performance of AMAGOLD on a variety of applications. We compare to full-batch baselines HMC and L2MC to show AMAGOLD is more efficient and we compare to SGHMC because, despite exhibiting bias, it is commonly used in the literature. Unless otherwise specified, we use reversible AMAGOLD, meaning we resample the momentum, $T=10$ and $\beta=0.25$. We set hyperparameters for AMAGOLD in a similar way as SGHMC \citep{chen2014stochastic}. For simplicity, we do not use the techniques in Section~\ref{sec:practice}. Additional details are in Appendix~\ref{sec:append:exp2}. The code can be found at \url{https://github.com/ruqizhang/amagold}.

\subsection{Synthetic Distributions}\label{sec:synthetic}
\begin{figure}[t!]
    \vspace{-0mm}
    \begin{tabular}{c c}        
        \hspace{-5mm}
        \centering
        \raisebox{0.1\height}{\includegraphics[height=2.5cm]{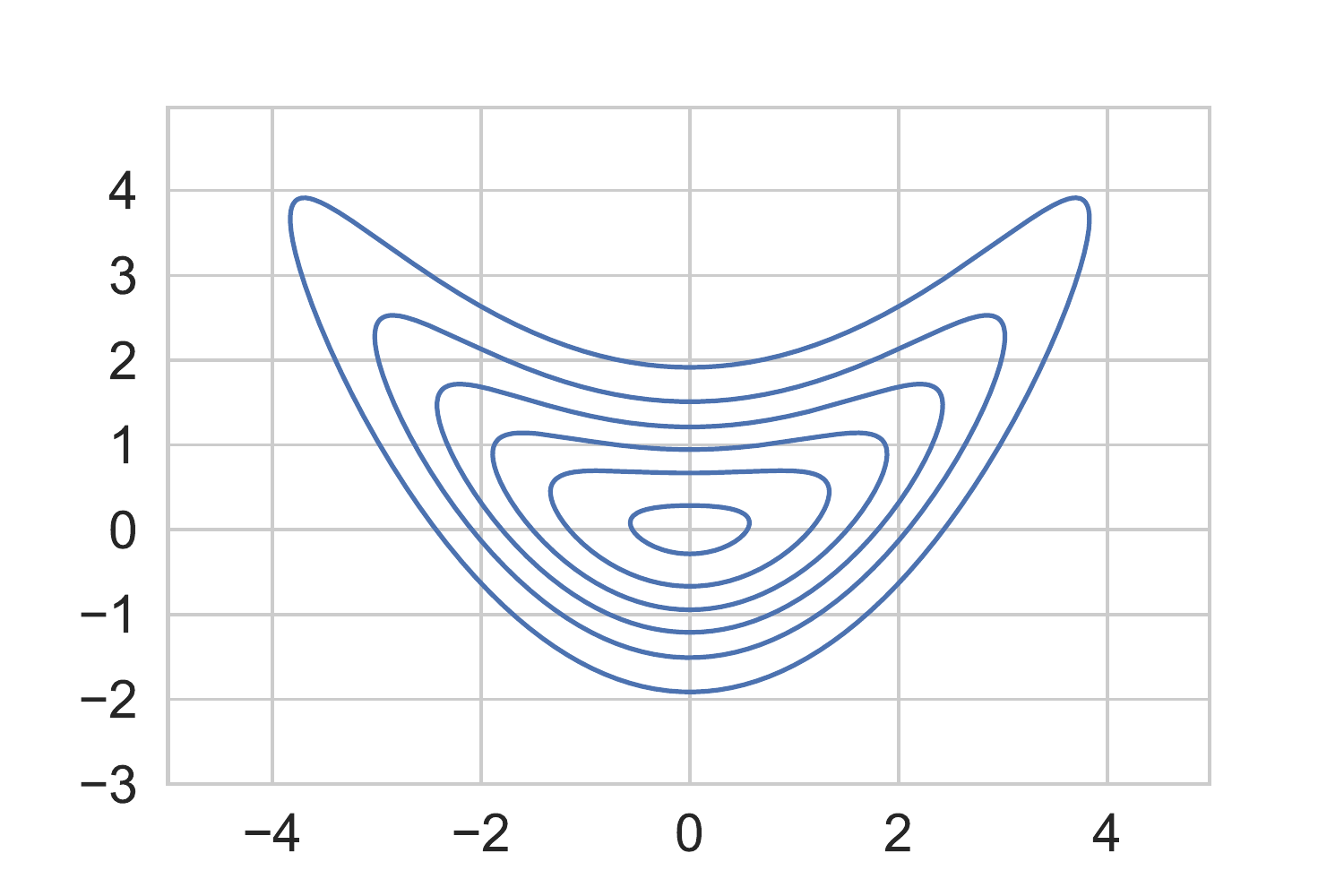}} &
        \hspace{-4mm}
        \includegraphics[width=4.2cm]{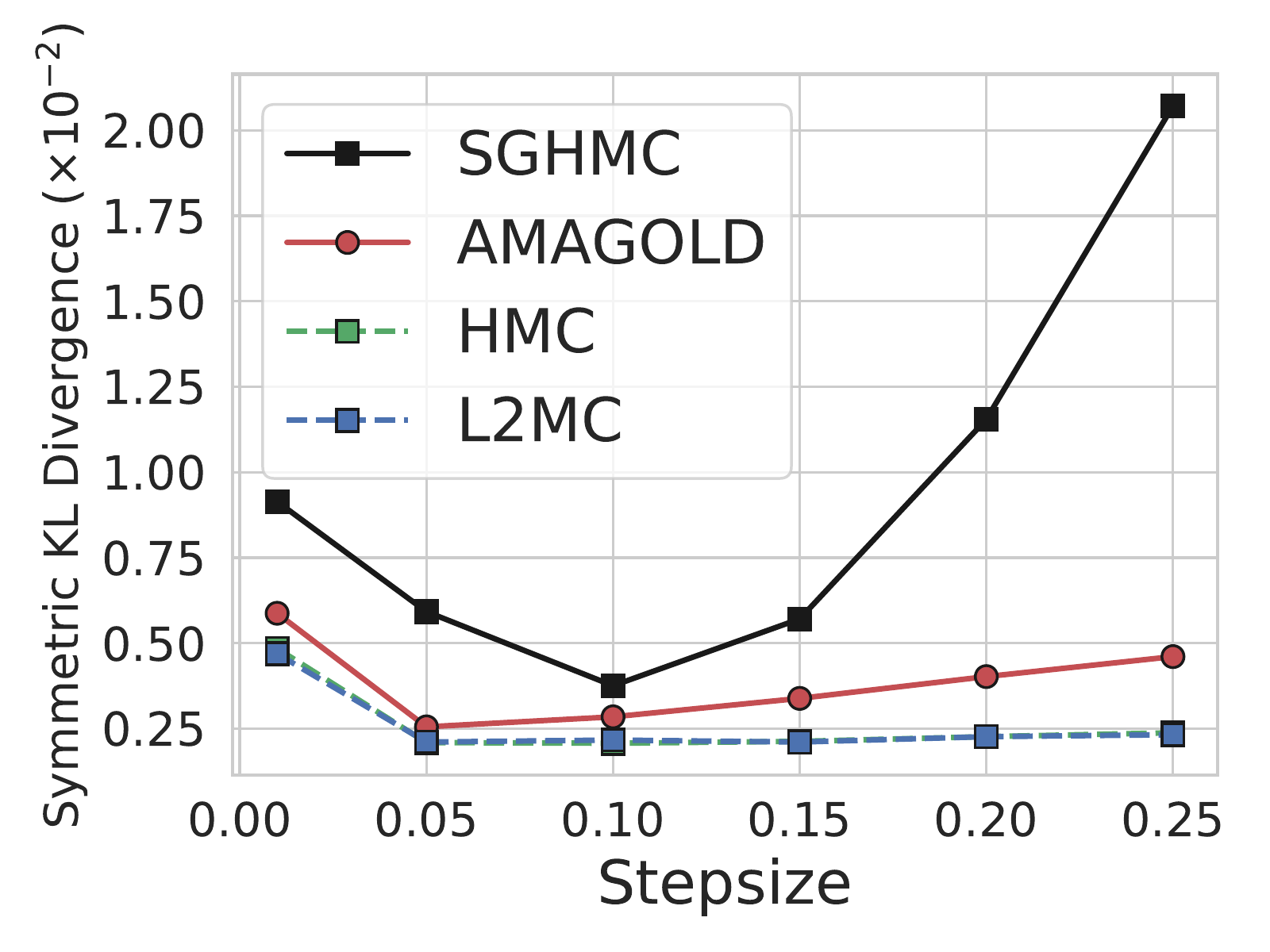}  \\
        (a) Dist1 \vspace{-0mm} & \hspace{-4mm}
        (b) KL comparison on Dist1 \hspace{-0mm} \hspace{-2mm}\\   
        \hspace{-5mm}
        \raisebox{0.1\height}{\includegraphics[height=2.5cm]{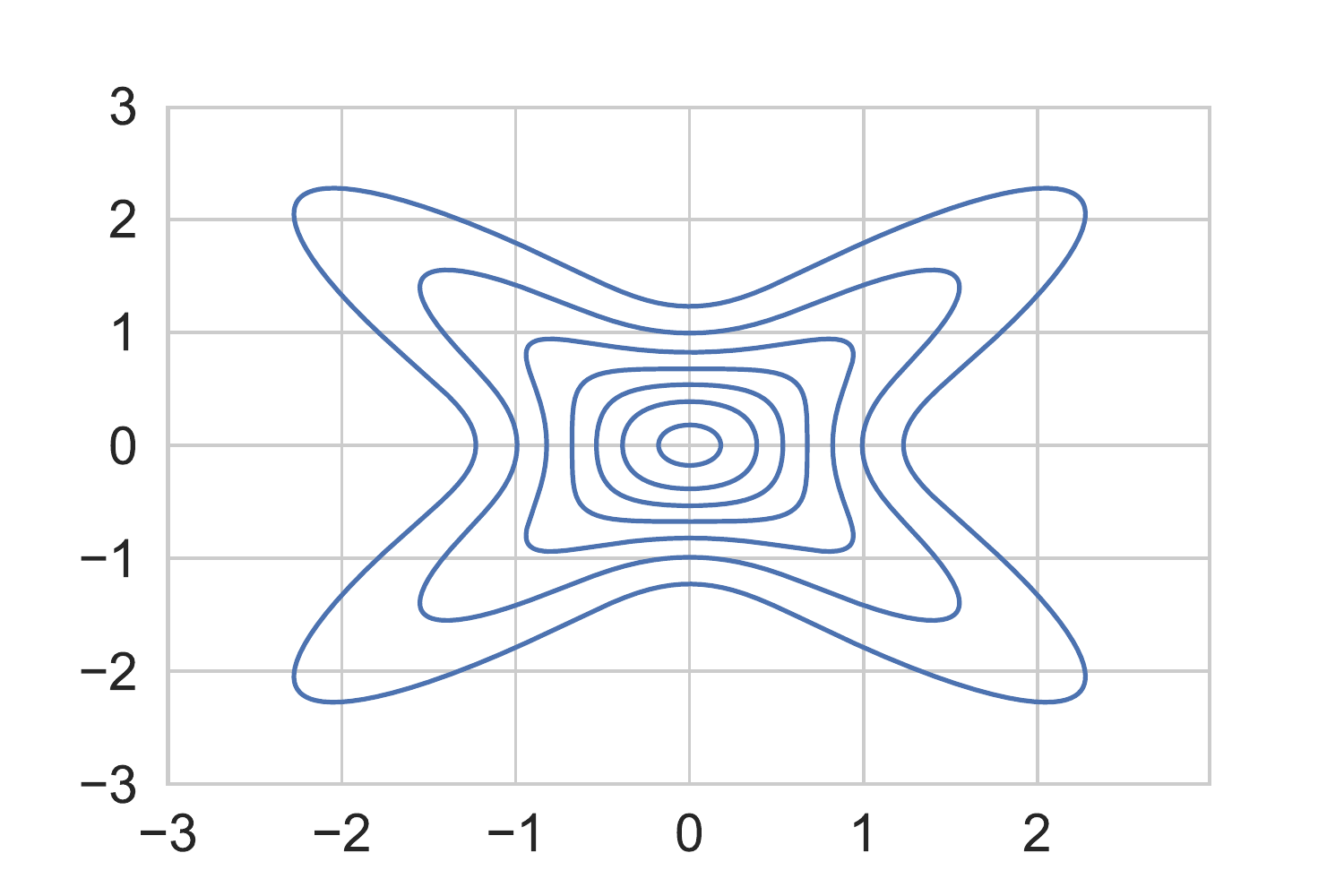}} &
        \hspace{-4mm}
        \includegraphics[width=4.2cm]{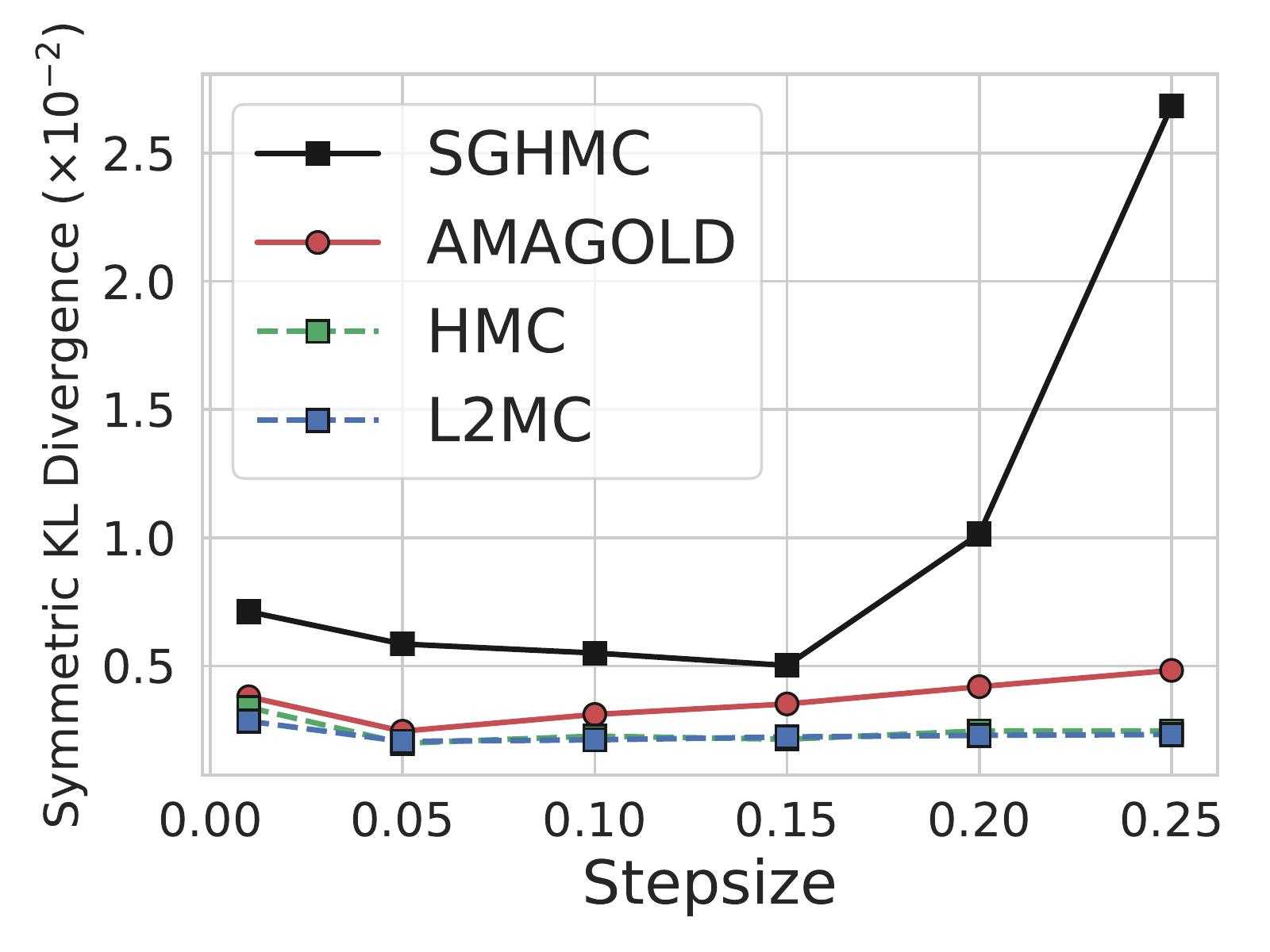} \\
        (c) Dist2 \vspace{-0mm} & \hspace{-4mm}
        (d) KL comparison on Dist2 \hspace{-0mm} \hspace{-2mm}\\   
    \end{tabular}
    \vspace{-0mm}
    \caption{AMAGOLD's performance against baselines. In (b) and (d) the step size varies from $0.01$ to $0.25$; the symmetric KL divergence is a function of step size.}
    \vspace{-0mm}
    \label{fig:synthetic}
\end{figure}
\vspace{-.1cm}
\begin{figure}[t!]
    \centering
    \includegraphics[width=5cm, height=3cm]{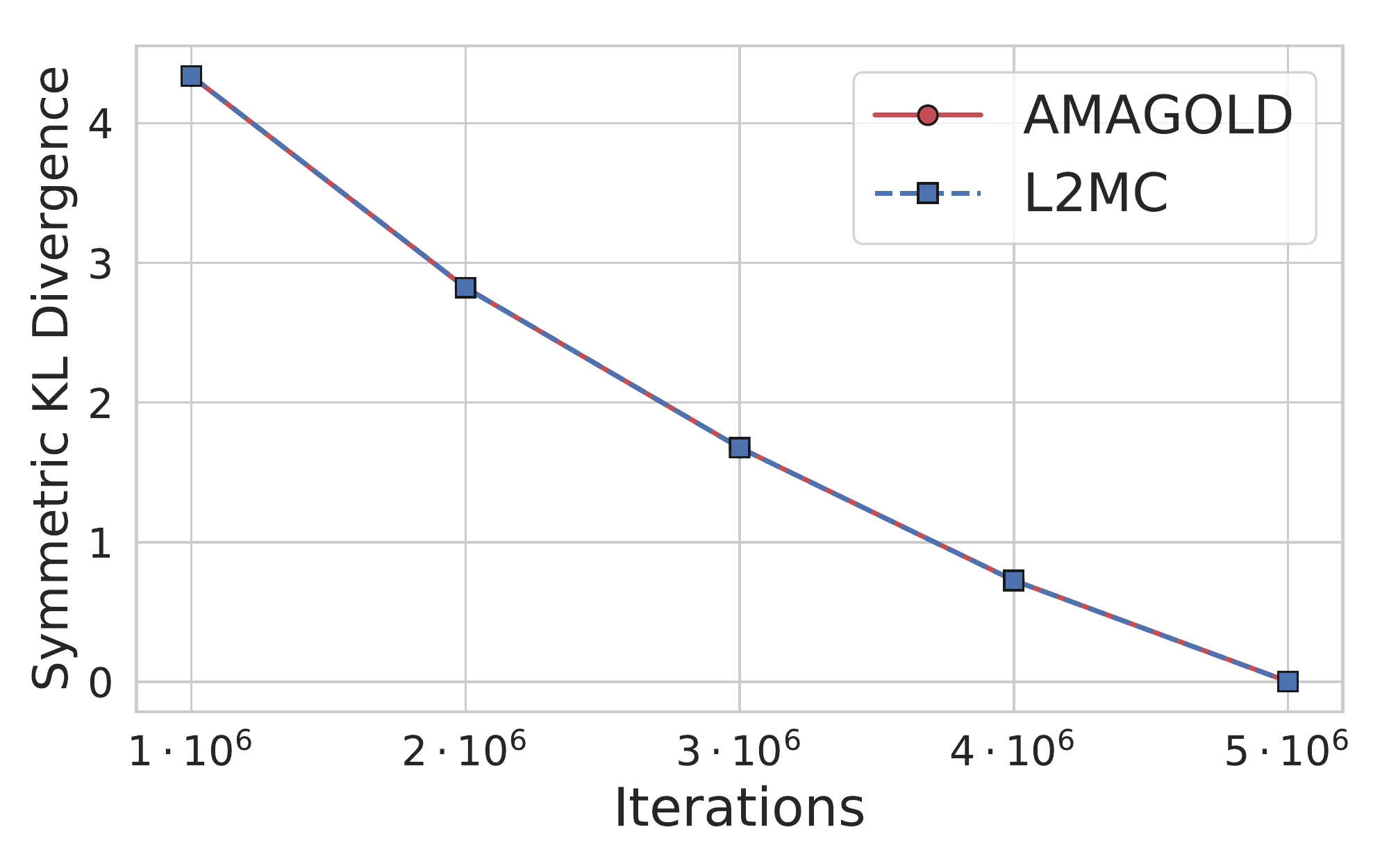} 
    \caption{The convergence speed (symmetric KL divergence as a function of iterations) of AMAGOLD compared to L2MC on Dist1 with step size 0.15.}
    \label{fig:syntheticcompare}
    \vspace{-5mm}
\end{figure}
We conduct experiments on synthetic two-dimensional distributions (Figures~\ref{fig:synthetic}a and~\ref{fig:synthetic}c), which are adapted from \citep{yin2018semi}. The analytical expressions are in Appendix~\ref{sec:append:dims}. We compare our algorithm against three baselines: (1) HMC, (2) L2MC with amortized M-H correction, and (3) SGHMC. HMC and L2MC serve as non-stochastic, unbiased baselines. As in \citet{chen2014stochastic}, we replace $\nabla U$ by stochastic estimates $\nabla \tilde U = \nabla U + \mathcal{N}(0, I)$ for the stochastic methods. We draw $5\times 10^6$ samples and use symmetric KL divergence as a function of step size to quantitatively evaluate the convergence of the Markov chain. On both distributions, AMAGOLD's symmetric KL divergence is close to full-batch methods and is much lower than SGHMC's, especially when the step size is large. This validates our theory that AMAGOLD is unbiased, while SGHMC's bias increases with step size. See Appendix \ref{sec:append:kl} for more details.

We then verify the theory that AMAGOLD has a comparable convergence rate to L2MC while using stochastic gradient estimates. Specifically, in Figure \ref{fig:syntheticcompare} AMAGOLD's convergence rate is the same as L2MC's (up to a constant factor slowdown of about $10^{-3}$). We include runtime comparisons in Appendix~\ref{sec:append:runtime}.

\subsection{Bayesian Logistic Regression on Real-World Data}\label{sec:rw}

We evaluate our method on Bayesian logistic regression using two real-world datasets: \emph{Australian}  and \emph{Heart} (Figure~\ref{fig:rw}). We compute the MSE between the estimated and true parameters, obtained from $10^7$ samples from HMC as in \citet{li2016preconditioned}. AMAGOLD exhibits smaller error than SGHMC on varying step sizes. We show runtime comparisons with step size $10^{-4}$. Compared to full-batch HMC and L2MC, AMAGOLD is significantly faster due to minibatching. It is also not much slower than SGHMC, indicating AMA can reduce the cost of adding the M-H step. AMAGOLD's large error using a large step size is due to a drop in M-H acceptance probability (Appendix~\ref{sec:append:lr}). However, this drop can be easily avoided in practice. One can either set the step size such that it achieves a reasonable acceptance rate (usually 20--80\%, depending on the application) or use the tuning technique in Section \ref{sec:practice}. With a reasonable acceptance rate, AMAGOLD achieves much lower error compared to SGHMC.

\begin{figure}[t]
    \vspace{-4mm}
    \begin{tabular}{c c}        
        \hspace{-.5cm}
        \centering
        \includegraphics[width=4cm]{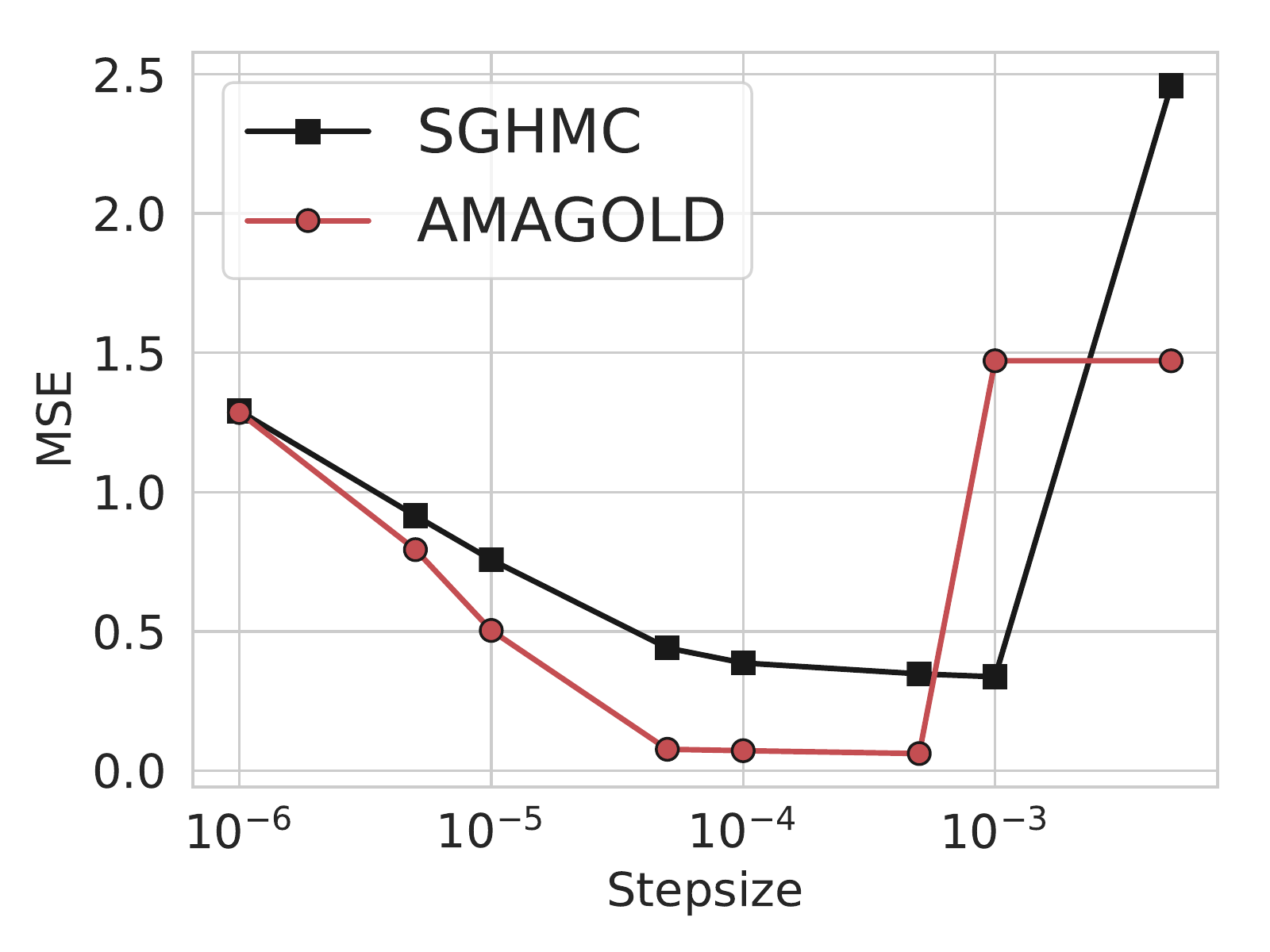} &
        \hspace{-6mm}
        \includegraphics[width=4cm]{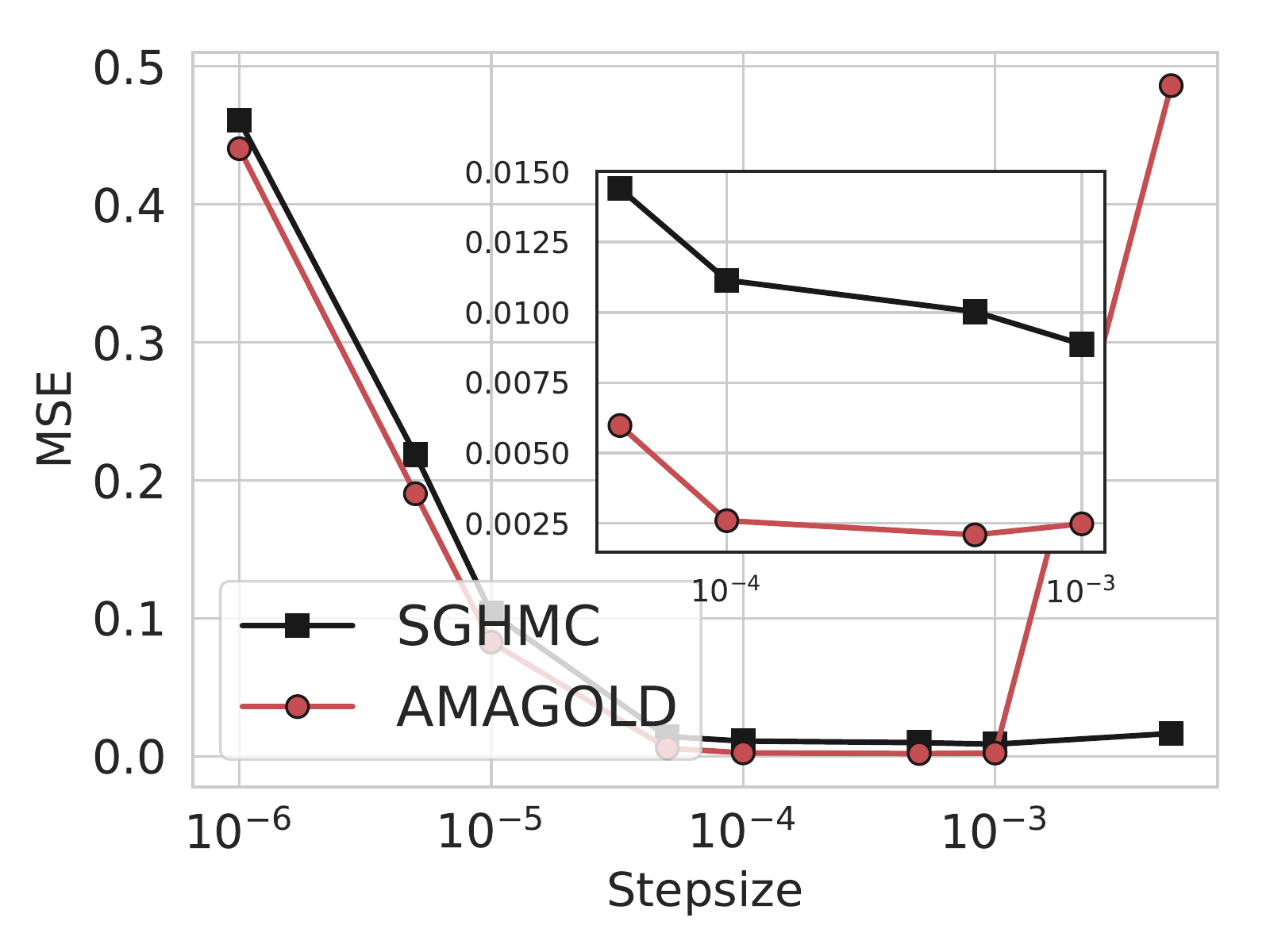}  \\
        (a) Australian \vspace{-0mm} & \hspace{-0mm}
        (b) Heart \hspace{-0mm} \hspace{-2mm}\\
        \hspace{-.5cm}
        \includegraphics[width=4cm]{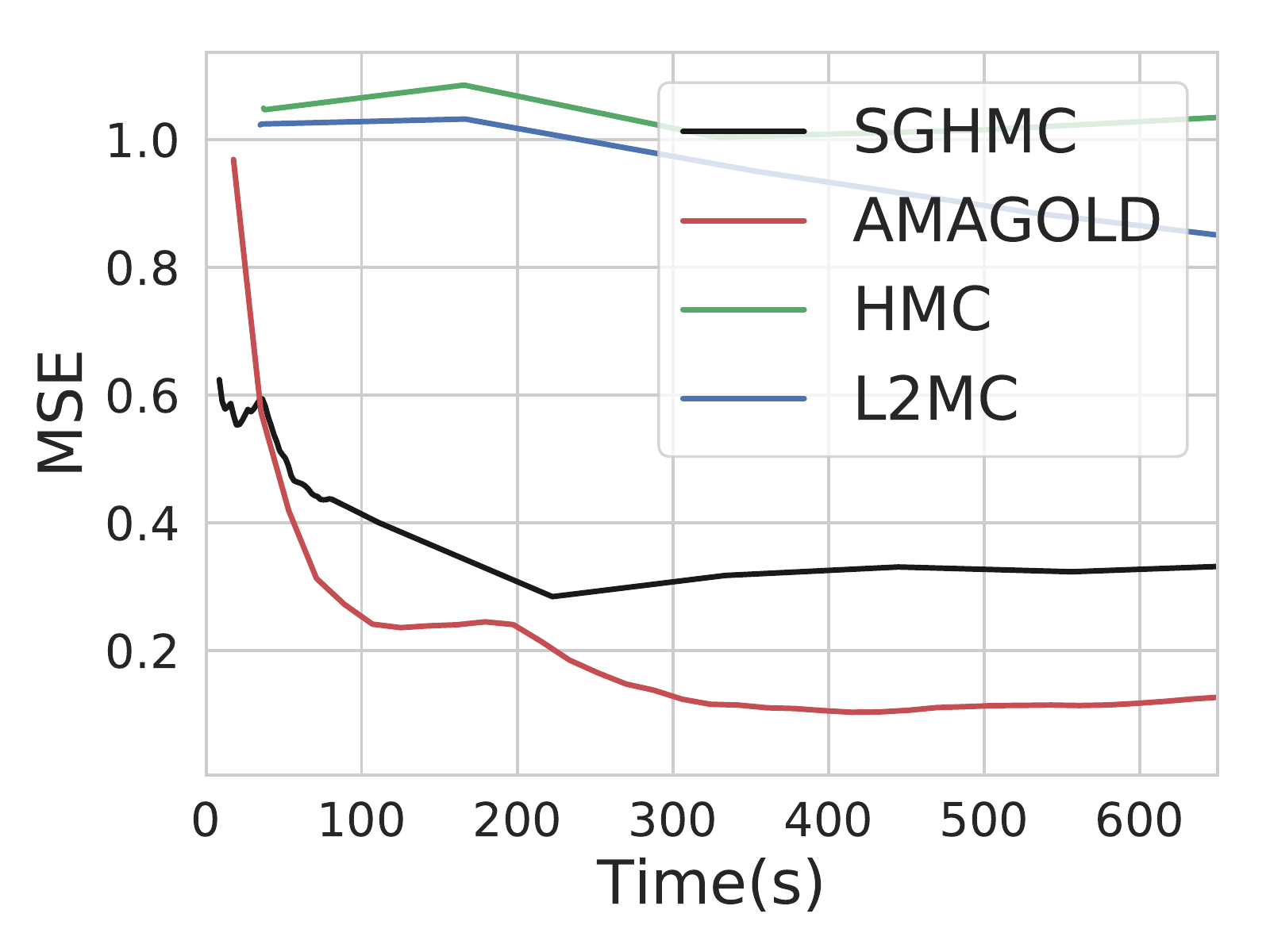} &
        \hspace{-4mm}
        \includegraphics[width=4cm]{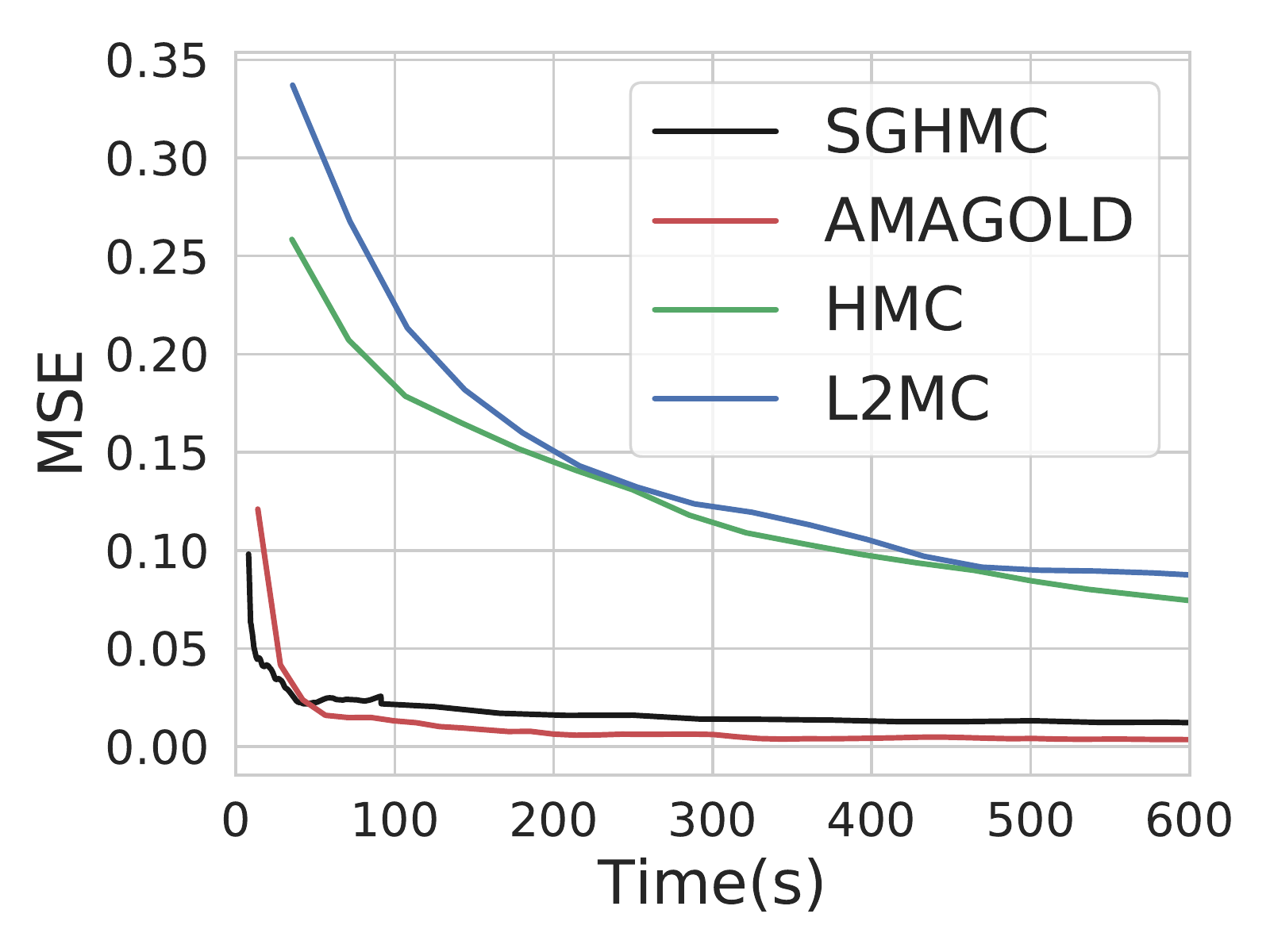} \\
        (c) Australian runtime \vspace{-0mm} & \hspace{-4mm}
        (d) Heart runtime \hspace{-0mm} \hspace{-2mm}\\ 
    \end{tabular}
    \vspace{-0mm}
    \caption{We use two real-world datasets (a) \emph{Australian} (15 covariates, 690 data points) and (b) \emph{Heart} (14 covariates, 270 data points). The minibatch size is 32 and 16, respectively. We collect $5\times 10^6$ samples and test step size varying from $10^{-6}$ to $5\times 10^{-3}$.}
    \label{fig:rw}
\end{figure}

\subsection{Bayesian Neural Networks}\label{sec:bnn}
We apply AMAGOLD on Bayesian neural networks. The architecture is a MLP with two-layer with RELU non-linearities. The dataset size is 60000 and we use minibatch size 2000. We use irreversible AMAGOLD since we find it gives better results. Similar to \citet{zhang2019cyclical}, to speed up the convergence of the sampling methods, we use SGD with momentum in the first 3 epochs as burn-in and then switch to either SGHMC or AMAGOLD.

\begin{table}
\centering
\begin{tabular}{@{}lrrr@{}} \toprule
Algorithm & $b$ & $h=0.0005$ & $h=0.001$\\ \midrule
SGHMC & 0.01  &3.69$\pm$0.03   &3.77$\pm$0.17    \\
SGHMC & 5e-6 &89.95$\pm$0.29    &89.70$\pm$0.91 \\ 
AMAGOLD & 0.01  &3.63$\pm$0.04   &3.65$\pm$0.08       \\
AMAGOLD & 5e-6 &3.65$\pm$0.10   &3.63$\pm$0.10 \\
\bottomrule
\end{tabular}
\caption{Comparison between AMAGOLD and SGHMC of test error (\%) $\pm$ standard error. We collect 20 samples in total.}
\label{tab:mnist}
\vspace{-.4cm}
\end{table}

\paragraph{Classification}
We evaluate the classification accuracy of AMAGOLD and SGHMC. As in \citet{chen2014stochastic}, we reparameterize our algorithm, setting $v=\epsilon\sigma^{-2}r, b=\epsilon\beta$ and $h=\epsilon^2\sigma^{-2}$ (Appendix~\ref{sec:append:reform}). This equivalent two-parameter reformulated update is similar to SGD with momentum and thus more easily tuned on DNNs. Table~\ref{tab:mnist} shows the test error on various hyperparameter settings. AMAGOLD yields consistent test error, regardless of the hyperparameter values. In contrast, the performance of SGHMC is affected significantly by the hyperparameters. When $b$ is small, SGHMC diverges. Similar performance of SGHMC has also been reported in \citet{ding2014bayesian}.

\paragraph{Uncertainty Evaluation}
We evaluate the sampling performance in terms of uncertainty evaluation, which is important in many ML applications \citep{lakshminarayanan2017simple,blundell2015weight}. We test predictive uncertainty estimation on out-of-distribution samples \citep{lakshminarayanan2017simple}. The 8 models in Table~\ref{tab:mnist} are tested on the notMNIST dataset \citep{notmnist}. Since the models have never seen the samples from notMNIST, ideally the predictive distribution should be uniform, which gives the maximum entropy. We plot the empirical CDF for the entropy of the predictive distribution (Figure~\ref{fig:uncertaintycdf}). AMAGOLD provides consistent uncertainty estimations on all settings, which aligns with the classification results. In contrast, when $b$ is small or $h$ is large, SGHMC performance suffers; it is overconfident about its prediction. 

Both of these experiments indicate that SGHMC is very sensitive to hyperparameters. It needs to be carefully tuned to achieve desired performance on classification and uncertainty estimation. In contrast, AMAGOLD is robust to various hyperparameter settings because it is guaranteed to converge to the target distribution.

\begin{figure}[ht!]
    \vspace{-2mm}
        \hspace{-5mm}
        \centering
        \includegraphics[width=6cm, height=3.75cm]{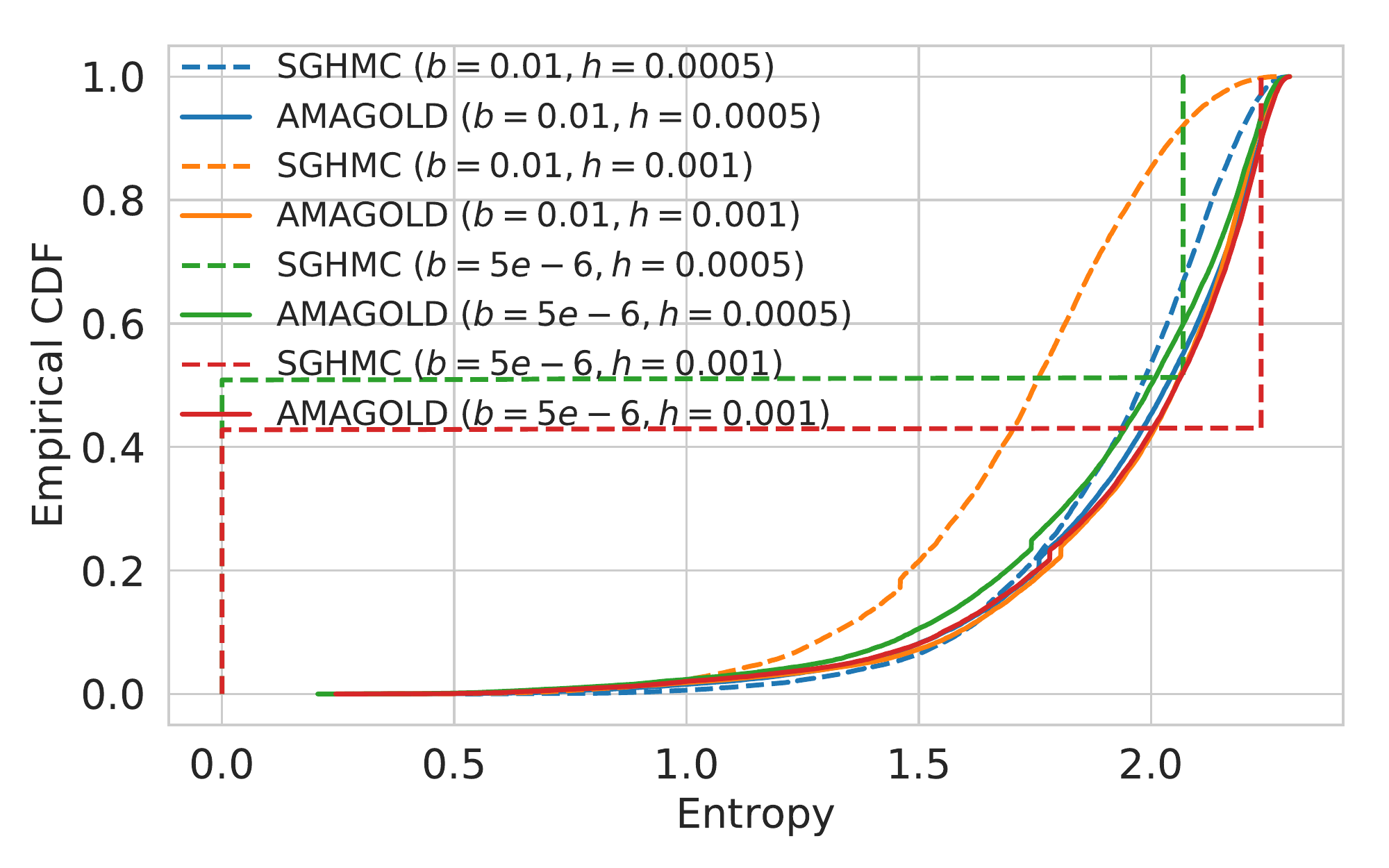}
        \hspace{-4mm}
    \vspace{-0mm}
    \caption{Empirical CDF on notMNIST dataset.}
    \vspace{-3mm}
    \label{fig:uncertaintycdf}
\end{figure}

\section{Conclusion}
Our work represents a first step toward unbiased, efficient second-order SG-MCMC. We introduced AMAGOLD, which achieves these goals by infrequently applying the computationally-expensive Metropolis-Hasting adjustment step, amortizing the cost across multiple algorithm steps. We prove this is sufficient for convergence to the target distribution, and provide reversible and non-reversible versions for practical use. AMAGOLD's convergence rate is theoretically guaranteed: the bound captures the trade-off between the speed-up from minibatching and the convergence rate. Lastly, our work is complementary to, rather than exclusive with, other research in stochastic MCMC. In future work it would be interesting to explore combining AMA with other SG-MCMC variants \citep{ding2014bayesian, zhang2017stochastic, ma2015complete} and minibatch M-H methods \citep{korattikara2014austerity, bardenet2014towards, maclaurin2015firefly, seita2016efficient}.

\section*{Acknowledgements}
This work was supported in part by Huawei Technologies Co., Ltd.
\bibliography{references}
\bibliographystyle{plainnat}

\newpage
\appendix
\onecolumn

{\Large\textbf{Appendix}}
\section{Background of SGHMC}\label{sec:append:sghmc}
SGHMC, Algorithm~\ref{alg:SGHMC}, reduces the computational cost of full-batch methods by using a stochastic gradient in lieu of $\nabla U$.
SGHMC estimates $U(\theta)$ by minibatch $\tilde{ \mathcal{D}}$. However, na\"ively replacing $U$ by $\tilde U$ in HMC will lead to divergence from the target distribution \citep{chen2014stochastic}. Therefore, SGHMC adds the additional friction term of L2MC to offset the noise introduced by using minibatch. That is, SGHMC is simulating the dynamics
\begin{align*}
d\theta = \sigma^{-2} r dt,
\end{align*}
\vspace{-2em}
\begin{align}
d r = - \nabla \tilde U (\theta)dt - 2\beta r dt + \mathcal{N}(0, 4(\beta-\hat\beta)\sigma^2 dt)   \label{eq:momentum}
\end{align}
where $\beta$ controls the friction term and $\hat\beta$ is an estimate of the stochastic gradient noise, which is often set to be zero in practice.

\section{Proof of Results in Section~\ref{sec:AMA}}\label{sec:append:AMA}

In this section, we will provide proofs of the results that we asserted in Section~\ref{sec:AMA} about reversibility and skew-reversibility of our algorithms.
First, for completeness we re-prove the fact that a skew-reversible chain has stationary distribution $\pi$, which is known, but not as well-known as the corresponding result for reversible chains.

\begin{lemma}
If $G$ is a skew-reversible chain, that is one that satisfies (\ref{eq:skew-detail-balance}), then $\pi$ is its stationary distribution.
\end{lemma}
\begin{proof}
Since $G$ is skew-reversibile, by definition it satisfies for any states $x$ and $y$ the conditions that $\pi(x) = \pi(x^{\bot})$ and
\[
\pi(x) G(x, y) = \pi(y^{\bot}) G(y^{\bot}, x^{\bot}).
\]
By combining these two we can easily get
\[
\pi(x) G(x, y) = \pi(y) G(y^{\bot}, x^{\bot}).
\]
Next, summing up over all $x$ in the while state space $\Omega$,
\[
\sum_{x \in \Omega} \pi(x) G(x, y) = \sum_{x \in \Omega} \pi(y) G(y^{\bot}, x^{\bot}) = \pi(y) \sum_{x \in \Omega} G(y^{\bot}, x^{\bot}).
\]
Since $\bot$ denotes an involution, it follows that summing up over $x$ for all $x$ in the state space is equal to summing up over all $x^{\bot}$, so
\[
\sum_{x \in \Omega} \pi(x) G(x, y) = \pi(y) \sum_{x \in \Omega} G(y^{\bot}, x) = \pi(y),
\]
where the last equality follows from the fact that for any Markov chain, the sum of the probabilities of transitioning into all states is always $1$.
So, we've shown that
\[
\sum_{x \in \Omega} \pi(x) G(x, y) = \pi(y)
\]
which can be written in matrix form as $\pi G = \pi$; this proves the lemma.
\end{proof}

\begin{lemma}
The amortized Metropolis-Hastings procedure described in Section~\ref{sec:AMA} using acceptance probability (\ref{eqn:amortmcmc}) results in a chain that is reversible with stationary distribution $\pi$.
\end{lemma}
\begin{proof}
According to the algorithm described in Section~\ref{sec:AMA}, the probability density of transitioning from state $x$ to state $y \ne x$ via intermediate states $x = x_0, x_1, x_2, \ldots, x_{T-1}, x_T = y$ is
\[
    \Exv{ \tau \cdot \prod_{t=0}^{T-1} P(x_t, x_{t+1}; \zeta_t) },
\]
where the expected value here is taken over the randomness used to select the stochastic samples $\zeta_t$.
This follows from the law of total expectation.
This means that the total probability of transitioning from $x$ to $y \ne x$ is just the integral of this over the intermediate states
\[
    G(x,y) = \int \Exv{ \tau \cdot \prod_{t=0}^{T-1} P(x_t, x_{t+1}; \zeta_t) } \; dx_1 \cdot dx_2 \cdots dx_{T-1},
\]
Now substituting in the value of $\tau$ from (\ref{eqn:amortmcmc}) gives us
\begin{align*}
    G(x,y)
    &=
    \int \Exv{ 
        \min\left(1, \frac{\pi(y)}{\pi(x)} \prod_{t = 0}^{T-1} \frac{P(x_{t+1},x_t; \zeta_t)}{P(x_t,x_{t+1}; \zeta_t)}\right) 
        \cdot
        \prod_{t=0}^{T-1} P(x_t, x_{t+1}; \zeta_t)
    } \; dx_1 \cdot dx_2 \cdots dx_{T-1} \\
    &=
    \int \Exv{ 
        \min\left(
            \prod_{t=0}^{T-1} P(x_t, x_{t+1}; \zeta_t), 
            \frac{\pi(y)}{\pi(x)} \prod_{t = 0}^{T-1} P(x_{t+1},x_t; \zeta_t)
        \right)
    } \; dx_1 \cdot dx_2 \cdots dx_{T-1}.
\end{align*}
Multiplying both sides by $\pi(x)$,
\begin{align*}
    \pi(x)  G(x,y)
    &=
    \int \Exv{ 
        \min\left(
            \pi(x) \prod_{t=0}^{T-1} P(x_t, x_{t+1}; \zeta_t), 
            \pi(y) \prod_{t = 0}^{T-1} P(x_{t+1},x_t; \zeta_t)
        \right)
    } \; dx_1 \cdot dx_2 \cdots dx_{T-1}.
\end{align*}
From here, the fact that $G$ is reversible follows directly from a substitution of $x_t \mapsto x_{T-t}$ in the integral, combined with the observation that the $\zeta_t$ are i.i.d. and so exchangeable.
\end{proof}

\begin{lemma}
The amortized Metropolis-Hastings procedure for skew-reversible chains described in Section~\ref{sec:AMA} using acceptance probability (\ref{eqn:amortmcmcskew}) results in a chain that is skew-reversible with stationary distribution $\pi$, as long as $\pi$ satisfies $\pi(x) = \pi(x^\bot)$.
\end{lemma}
\begin{proof}
As above, the probability density of transitioning from state $x$ to state $y \ne x$ via intermediate states $x = x_0, x_1, x_2, \ldots, x_{T-1}, x_T = y$ is
\[
    \Exv{ \tau \cdot \prod_{t=0}^{T-1} P(x_t, x_{t+1}; \zeta_t) },
\]
where the expected value here is taken over the randomness used to select the stochastic samples $\zeta_t$.
This follows from the law of total expectation.
This means that the total probability of transitioning from $x$ to $y \ne x$ is just the integral of this over the intermediate states
\[
    G(x,y) = \int \Exv{ \tau \cdot \prod_{t=0}^{T-1} P(x_t, x_{t+1}; \zeta_t) } \; dx_1 \cdot dx_2 \cdots dx_{T-1},
\]
Now substituting in the value of $\tau$ from (\ref{eqn:amortmcmc}) gives us
\begin{align*}
    G(x,y)
    &=
    \int \Exv{ 
        \min\left(1, \frac{\pi(y)}{\pi(x)} \prod_{t = 0}^{T-1} \frac{P(x_{t+1}^\bot,x_t^\bot; \zeta_t)}{P(x_t,x_{t+1}; \zeta_t)}\right) 
        \cdot
        \prod_{t=0}^{T-1} P(x_t, x_{t+1}; \zeta_t)
    } \; dx_1 \cdot dx_2 \cdots dx_{T-1} \\
    &=
    \int \Exv{ 
        \min\left(
            \prod_{t=0}^{T-1} P(x_t, x_{t+1}; \zeta_t), 
            \frac{\pi(y)}{\pi(x)} \prod_{t = 0}^{T-1} P(x_{t+1}^\bot,x_t^\bot; \zeta_t)
        \right)
    } \; dx_1 \cdot dx_2 \cdots dx_{T-1}.
\end{align*}
Multiplying both sides by $\pi(x)$, and leveraging the fact that $\pi(x) = \pi(x^\bot)$,
\begin{align*}
    \pi(x)  G(x,y)
    &=
    \int \Exv{ 
        \min\left(
            \pi(x) \prod_{t=0}^{T-1} P(x_t, x_{t+1}; \zeta_t), 
            \pi(y^\bot) \prod_{t = 0}^{T-1} P(x_{t+1}^\bot,x_t^\bot; \zeta_t)
        \right)
    } \; dx_1 \cdot dx_2 \cdots dx_{T-1}.
\end{align*}
From here, the fact that $G$ is skew-reversible follows directly from a substitution of $x_t \mapsto x_{T-t}^\bot$ in the integral (which is a valid substitution without introducing an extra constant term because the involution $\bot$ is measure-preserving by assumption), combined with the observation that the $\zeta_t$ are i.i.d. and so exchangeable.
\end{proof}

\begin{lemma}
A skew-reversible chain will become reversible by resampling the momentum at the beginning of outer loop.
\end{lemma}
\begin{proof}
Assume the chain starts at $(\theta, r)$ and ends at $(\theta^*, r^*)$. By the skew-detailed balance, we have 
\[
\pi(\theta, r) G((\theta, r), (\theta^*, r^*)) = \pi(\theta^*, -r^*) G((\theta^*, -r^*), (\theta, -r))
\]
Since the momentum is resampled and is independent of $\theta$, we can integrate it and describe the chain in terms of $\theta$
\begin{align*}
    \pi(\theta)H(\theta,\theta^*) := \int \pi(\theta) \pi(r)G((\theta, r), (\theta^*, r^*))drdr^*
\end{align*}
Similarly, we have
\begin{align*}
    \pi(\theta^*)H(\theta^*,\theta) := \int \pi(\theta^*) \pi(-r^*)G((\theta^*, -r^*),(\theta, -r))dr^*dr
\end{align*}
By the skew-detailed balance we know
\[
\pi(\theta)H(\theta,\theta^*) = \pi(\theta^*)H(\theta^*,\theta)
\]
This proves the lemma.
\end{proof}

\section{Proof of Theorem \ref{thm:convergence}}\label{sec:thm1}
\begin{proof}
First, we consider resampling momentum in Algorithm~\ref{alg:IMA} and will show that the chain is reversible. We consider the probability of starting from $\theta$ and going through a particular sequence of $r_{t + \frac{1}{2}}$ and $\theta_t$ and arriving at $(\theta^*, r^*)$. We have $G(\theta, \theta^*)$, which is the transition probability from $\theta$ to $\theta^*$, as the following
\[
G(\theta, \theta^*) = \mathbf{E}\int \Probc{\theta_0, \theta_1, \ldots, \theta_{T-1}, \theta^*}{\theta, \tilde U_0, \ldots, \tilde U_{T-1}} \min(1, a({\bm \theta}))d\theta_0\cdots d\theta_{T-1}
\]
where ${\bm \theta} = \{\theta_0,\ldots, \theta_{T-1}, \theta^*\}$ and the expectation is taken over the stochastic energy function samples $\tilde U_t$.

Next, we want to derive the probability density in terms of $r$ and ${\bm \eta} = \{\eta_0,\ldots,\eta_{T-1}\}$. This involves a change of variables in the PDF formula. We notice that ${\bm \theta}$ is a bijective function of $r$ and ${\bm \eta}$. By the rule of change of variables, we know that
\begin{align*}
&\hspace{2em}\Probc{\theta_0, \theta_1, \ldots, \theta_{T-1}, \theta^*}{\theta, \tilde U_0, \ldots, \tilde U_{T-1}} \min(1, a({\bm \theta})) \\
&= \Probc{r, \eta_0, \ldots, \eta_{T-1}}{\theta, \tilde U_0, \ldots, \tilde U_{T-1}} \min(1, a({\bm \eta, r)})) det^{-1}(D_{({\bm \eta,r})} ({\bm \theta}))
\end{align*}
where $D_{({\bm \eta,r})} ({\bm \theta}, \theta^*)$ is the Jacobian matrix. 

To get this Jacobian matrix, we first apply the chain rule,
\begin{align*}
  D_{({\bm \eta,r})} ({\bm \theta}) = D_{{\bm r} } ({\bm \theta}) \cdot D_{({\bm \eta,r})} {\bm r}
\end{align*}
where ${\bm r} = \{r, r ,\ldots,r_{T-\frac{1}{2}}\}$.

Since the derivative of $\theta_t$ with respect to any $r_{s - \frac{1}{2}}$ for $s > t$ is zero, it follows that $D_{{\bm r} } ({\bm \theta})$ will be triangular, and so the determinant is just the product of the diagonal entries.
From our formula for the update rule,
\begin{align*}
    \theta_t &= \theta + \frac{1}{2}\epsilon\sigma^{-2}r_{t-\frac{1}{2}},\text{ for }t=0,T\\
    \theta_t &= \theta_{t-1} + \epsilon\sigma^{-2}r_{-\frac{1}{2}},\text{ for }t=1,\ldots,T-1.
\end{align*}
Therefore,
\begin{align*}
    \frac{\partial \theta_0}{\partial r_{t-\frac{1}{2}}} &= \frac{1}{2}\epsilon\sigma^{-2}I_d,\text{ for }t=0,T\\
    \frac{\partial \theta_t}{\partial r_{t-\frac{1}{2}}} &= \epsilon\sigma^{-2}I_d,\text{ for }t=1,\ldots,T-1.
\end{align*}
It follows that
\[
  \det \left(D_{{\bm r} } ({\bm \theta})\right) = \frac{1}{4^d}\left(\epsilon\sigma^{-2}\right)^{(T+1)d}.
\]
Similarly, the derivative of $\eta_t$ with respect to any $r_{s-\frac{1}{2}}$ for $s > t$ is zero, it follows that $D_{({\bm \eta,r})} {\bm r}$ will be triangular, and so the determinant is just the product of the diagonal entries.
That is,
\[
  D_{{\bm r}} ({\bm \eta,r})) = \prod_{t = 0}^{T-1} \frac{\partial \eta_t}{\partial r_{t+\frac{1}{2}}}.
\]
From our original formula for the update rule,
\[
  r_{t + \frac{1}{2}} = r_{t - \frac{1}{2}} - \epsilon \nabla \tilde U_t(\theta_t) - \epsilon \beta \left(r_{t - \frac{1}{2}} + r_{t + \frac{1}{2}}\right) + \eta_t
\]
we have
\[
  (1 + \epsilon \beta) r_{t + \frac{1}{2}} = r_{t - \frac{1}{2}} - \epsilon \nabla  \tilde U_t(\theta_t) - \epsilon \beta r_{t - \frac{1}{2}} + \eta_t,
\]
and so
\[
  \frac{\partial \eta_t}{\partial r_{t+\frac{1}{2}}}
  =
  (1 + \epsilon \beta)I_d.
\]
It follows that
\[
  \det \left(D_{({\bm \eta,r})} {\bm r}\right)
  =
  (1 + \epsilon \beta)^{-Td }
\]

Now we can get that
\begin{align*}
  \det\left( D_{({\bm \eta,r})} ({\bm \theta})\right) &= \det\left( D_{{\bm r} } ({\bm \theta}) \right)\cdot \det\left( D_{({\bm \eta,r})} {\bm r}\right) \\
  &=(1 + \epsilon \beta)^{-Td } \cdot \frac{1}{4^d}\left(\epsilon\sigma^{-2}\right)^{(T+1)d}
\end{align*}
Thus, 
\begin{align*}
G(\theta, \theta^*) &= \mathbf{E}\int \Probc{\theta_0, \theta_1, \ldots, \theta_{T-1}, \theta^*}{\theta, \tilde U_0, \ldots, \tilde U_{T-1}} \min(1, a({\bm \theta}))d\theta_0\cdots d\theta_{T-1}\\
&=
(1 + \epsilon \beta)^{Td } \cdot 4^d\left(\epsilon\sigma^{-2}\right)^{-(T+1)d}\mathbf{E}\int \Probc{r, \eta_0, \ldots, \eta_{T-1}}{\theta, \tilde U_0, \ldots, \tilde U_{T-1}} 
\\&\hspace{2em}\min(1, a({\bm \eta, r)})) d\theta_0\cdots d\theta_{T-1}
\end{align*}

By the distribution of $r$ and $\eta_t$, we know that
\begin{align*}
&\hspace{2em}\Probc{r, \eta_0, \eta_2, \ldots, \eta_{T - 2}}{\theta, \tilde U_0, \ldots, \tilde U_{T-1}}\\
  &=
  \left(2 \pi \sigma^2 \right)^{\frac{-d}{2}}
  \cdot
  \exp\left( -\frac{\norm{r}^2}{2 \sigma^2} \right)
  \cdot
  \prod_{t = 0}^{T - 1}
  \left(8 \pi \epsilon \beta \sigma^2 \right)^{\frac{-d}{2}}
  \cdot
  \exp\left( -\frac{\norm{\eta_t}^2}{8 \epsilon \beta \sigma^2} \right) \\
  &=
  \left(2 \pi \sigma^2 \right)^{\frac{-d}{2}}
  \cdot
  \left(8 \pi \epsilon \beta \sigma^2 \right)^{\frac{-T d}{2}}
  \cdot
  \exp\left( -\frac{\norm{r}^2}{2 \sigma^2} \right)
  \cdot
  \exp\left( -\frac{1}{8 \epsilon \beta \sigma^2} \sum_{t = 0}^{T - 1} \norm{\eta_t}^2\right).
\end{align*}

Notice that
\begin{align*}
  \sum_{t = 0}^{T - 1} \norm{\eta_t}^2
  &=
  \sum_{t = 0}^{T - 1} \norm{r_{t + \frac{1}{2}} - r_{t - \frac{1}{2}} + \epsilon   \nabla \tilde U_t(\theta_t) + \epsilon \beta \left(r_{t - \frac{1}{2}} + r_{t + \frac{1}{2}}\right)}^2 \\
  &=
  \sum_{t = 0}^{T - 1}
  \norm{r_{t + \frac{1}{2}} - r_{t - \frac{1}{2}} + \epsilon   \nabla \tilde U_t(\theta_t)}^2
  \\&\hspace{2em}+ 
  2 \epsilon \beta \left( r_{t + \frac{1}{2}} - r_{t - \frac{1}{2}} + \epsilon   \nabla \tilde U_t(\theta_t) \right)^T \left(r_{t - \frac{1}{2}} + r_{t + \frac{1}{2}}\right)
  +
  \epsilon^2 \beta^2 \norm{ r_{t - \frac{1}{2}} + r_{t + \frac{1}{2}} }^2 \\
  &=
  \sum_{t = 0}^{T - 1}
  \norm{r_{t + \frac{1}{2}} - r_{t - \frac{1}{2}} + \epsilon   \nabla \tilde U_t(\theta_t)}^2
  + 
  2 \epsilon \beta \left( \norm{ r_{t + \frac{1}{2}} }^2 - \norm{ r_{t - \frac{1}{2}} }^2 \right)
  \\&\hspace{2em}+
  2 \epsilon^2 \beta   \nabla \tilde U_t(\theta_t)^T \left(r_{t - \frac{1}{2}} + r_{t + \frac{1}{2}}\right)
  +
  \epsilon^2 \beta^2 \norm{ r_{t - \frac{1}{2}} + r_{t + \frac{1}{2}} }^2 \\
  &=
  2 \epsilon \beta \left( \norm{ r_{T - \frac{1}{2}} }^2 - \norm{ r}^2 \right)
  +
  \sum_{t = 0}^{T - 1}
  \norm{r_{t + \frac{1}{2}} - r_{t - \frac{1}{2}} + \epsilon   \nabla \tilde U_t(\theta_t)}^2
  \\&\hspace{2em}+
  4 \epsilon \beta \sigma^2 \left( \rho_{t + \frac{1}{2}} - \rho_{t - \frac{1}{2}} \right)
  +
  \epsilon^2 \beta^2 \norm{ r_{t - \frac{1}{2}} + r_{t + \frac{1}{2}} }^2 \\
  &=
  2 \epsilon \beta \left( \norm{ r^* }^2 - \norm{ r }^2 \right)
  +
  4 \epsilon \beta \sigma^2 \left( \rho_{T - \frac{1}{2}} - \rho_{-\frac{1}{2}} \right)
  \\&\hspace{2em}+
  \sum_{t = 0}^{T - 1}
  \norm{r_{t + \frac{1}{2}} - r_{t - \frac{1}{2}} + \epsilon   \nabla \tilde U_t(\theta_t)}^2
  +
  \epsilon^2 \beta^2 \norm{ r_{t - \frac{1}{2}} + r_{t + \frac{1}{2}} }^2.
\end{align*}

By substituting this above and recalling that $\rho_{-\frac{1}{2}} = 0$, 
\begin{align*}
G(\theta, \theta^*)
  &=
  (1 + \epsilon \beta)^{Td } \cdot 4^d\left(\epsilon\sigma^{-2}\right)^{-(T+1)d}\mathbf{E}\int \Probc{r, \eta_0, \eta_2, \ldots, \eta_{T - 2},\theta^*}{\theta, \tilde U_0, \ldots, \tilde U_{T-1}}
  \\&\hspace{2em}\min(1, a({\bm \eta}, r))d\theta_0\cdots d\theta_{T-1}\\
  &=
 (1 + \epsilon \beta)^{Td } \cdot 4^d\left(\epsilon\sigma^{-2}\right)^{-(T+1)d}\mathbf{E}\int
  \left(2 \pi \sigma^2 \right)^{\frac{-d}{2}}
  \cdot
  \left(8 \pi \epsilon \beta \sigma^2 \right)^{\frac{-T d}{2}}
  \cdot
  \exp\left( -\frac{\norm{r}^2}{2 \sigma^2} \right)
  \\&\hspace{2em}\cdot
  \exp\left( -\frac{1}{8 \epsilon \beta \sigma^2} \cdot 2 \epsilon \beta \left( \norm{ r^* }^2 - \norm{ r }^2 \right) \right)
  \\&\hspace{2em}\cdot
  \exp\left( -\frac{1}{8 \epsilon \beta \sigma^2} \cdot 4 \epsilon \beta \sigma^2 \left( \rho_{T - \frac{1}{2}} - \rho_{-\frac{1}{2}} \right) \right)
  \\&\hspace{2em}\cdot
  \exp\left( -\frac{1}{8 \epsilon \beta \sigma^2} \cdot \sum_{t = 0}^{T - 1}
  \norm{r_{t + \frac{1}{2}} - r_{t - \frac{1}{2}} + \epsilon   \nabla \tilde U_t(\theta_t)}^2
  +
  \epsilon^2 \beta^2 \norm{ r_{t - \frac{1}{2}} + r_{t + \frac{1}{2}} }^2 \right)
  \\&\hspace{2em}\cdot\min(1, a)d\theta_0\cdots d\theta_{T-1} \\
  &=
  (1 + \epsilon \beta)^{Td } \cdot 4^d\left(\epsilon\sigma^{-2}\right)^{-(T+1)d}\mathbf{E}\int
  \left(2 \pi \sigma^2 \right)^{\frac{-d}{2}}
  \cdot
  \left(8 \pi \epsilon \beta \sigma^2 \right)^{\frac{-T d}{2}}
  \\&\hspace{2em}\cdot
  \exp\left( -\frac{1}{4 \sigma^2} \left( \norm{ r^* }^2 + \norm{ r }^2 \right) \right)
  \cdot
  \exp\left( -\frac{1}{2} \rho_{T - \frac{1}{2}} \right)
  \\&\hspace{2em}\cdot
  \exp\left( -\frac{1}{8 \epsilon \beta \sigma^2} \cdot \sum_{t = 0}^{T - 1}
  \norm{r_{t + \frac{1}{2}} - r_{t - \frac{1}{2}} + \epsilon   \nabla \tilde U_t(\theta_t)}^2
  +
  \epsilon^2 \beta^2 \norm{ r_{t - \frac{1}{2}} + r_{t + \frac{1}{2}} }^2 \right)
  \\&\hspace{2em}\cdot\min(1, a)d\theta_0\cdots d\theta_{T-1} \\
\end{align*}
where ${\bm r}$ are to be understood as functions of the $\theta_t$, and the integral is taken over $\theta_t$.

Substituting the expression of $a$, then the term inside the integral is
\begin{align*}
  &\left(2 \pi \sigma^2 \right)^{\frac{-d}{2}}
  \cdot
  \left(8 \pi \epsilon \beta \sigma^2 \right)^{\frac{-T d}{2}}
  \cdot
  \exp\left( -\frac{1}{4 \sigma^2} \left( \norm{ r^*}^2 + \norm{ r }^2 \right) \right)
  \\&\hspace{2em}\cdot
  \exp\left( -\frac{1}{2} \rho_{T - \frac{1}{2}} \right)
  \\&\hspace{2em}\cdot
  \exp\left( -\frac{1}{8 \epsilon \beta \sigma^2} \cdot \sum_{t = 0}^{T - 1}
  \norm{r_{t + \frac{1}{2}} - r_{t - \frac{1}{2}} + \epsilon   \nabla \tilde U_t(\theta_t)}^2
  +
  \epsilon^2 \beta^2 \norm{ r_{t - \frac{1}{2}} + r_{t + \frac{1}{2}} }^2 \right)
  \\&\hspace{2em}\cdot
  \min\left(1, \exp\left( U(\theta) - U(\theta^*) + \rho_{T - \frac{1}{2}} \right) \right) \\
  &=
  \left(2 \pi \sigma^2 \right)^{\frac{-d}{2}}
  \cdot
  \left(8 \pi \epsilon \beta \sigma^2 \right)^{\frac{-T d}{2}}
  \cdot
  \exp\left( -\frac{1}{4 \sigma^2} \left( \norm{ r^* }^2 + \norm{ r}^2 \right) \right)
  \\&\hspace{2em}\cdot
  \exp\left( -\frac{1}{8 \epsilon \beta \sigma^2} \cdot \sum_{t = 0}^{T - 1}
  \norm{r_{t + \frac{1}{2}} - r_{t - \frac{1}{2}} + \epsilon   \nabla \tilde U_t(\theta_t)}^2
  +
  \epsilon^2 \beta^2 \norm{ r_{t - \frac{1}{2}} + r_{t + \frac{1}{2}} }^2 \right)
  \\&\hspace{2em}\cdot
  \exp\left( U(\theta) \right)
  \cdot
  \min\left(\exp\left( -U(\theta) - \frac{1}{2} \rho_{T - \frac{1}{2}} \right), \exp\left( - U(\theta^*) + \frac{1}{2} \rho_{T - \frac{1}{2}} \right) \right).
\end{align*}
Finally, this probability multiplied by the probability of $\theta_0$, which is $\frac{1}{Z} \exp\left( -U(\theta) \right)$, is
\begin{align*}
  &\pi(\theta) G(\theta, \theta^*)\\
  &=\frac{1}{Z} \cdot (1 + \epsilon \beta)^{Td } \cdot 4^d\left(\epsilon\sigma^{-2}\right)^{-(T+1)d}
  \\&\hspace{2em}\mathbf{E}\int \left(2 \pi \sigma^2 \right)^{\frac{-d}{2}}
  \cdot
  \left(8 \pi \epsilon \beta \sigma^2 \right)^{\frac{-T d}{2}}
  \cdot
  \exp\left( -\frac{1}{4 \sigma^2} \left( \norm{ r^* }^2 + \norm{ r}^2 \right) \right)
  \\&\hspace{2em}\cdot
  \exp\left( -\frac{1}{8 \epsilon \beta \sigma^2} \cdot \sum_{t = 0}^{T - 1}
  \norm{r_{t + \frac{1}{2}} - r_{t - \frac{1}{2}} + \epsilon   \nabla \tilde U_t(\theta_t)}^2
  +
  \epsilon^2 \beta^2 \norm{ r_{t - \frac{1}{2}} + r_{t + \frac{1}{2}} }^2 \right)
  \\&\hspace{2em}\cdot
  \exp\left( U(\theta) \right)
  \cdot
  \min\left(\exp\left( -U(\theta) - \frac{1}{2} \rho_{T - \frac{1}{2}} \right), \exp\left( - U(\theta^*) + \frac{1}{2} \rho_{T - \frac{1}{2}} \right) \right) d\theta_0\cdots d\theta_{T-1}\\
  &=
  \frac{1}{Z} \cdot (1 + \epsilon \beta)^{Td } \cdot 4^d\left(\epsilon\sigma^{-2}\right)^{-(T+1)d} \cdot \left(2 \pi \sigma^2 \right)^{\frac{-d}{2}}
  \cdot
  \left(8 \pi \epsilon \beta \sigma^2 \right)^{\frac{-T d}{2}}
  \cdot
  \\&\hspace{2em}\mathbf{E}\int \exp\left( -\frac{1}{4 \sigma^2} \left( \norm{ r^* }^2 + \norm{ r }^2 \right) \right)
  \\&\hspace{2em}\cdot
  \exp\left( -\frac{1}{8 \epsilon \beta \sigma^2} \cdot \sum_{t = 0}^{T - 1} \left(
  \norm{r_{t + \frac{1}{2}} - r_{t - \frac{1}{2}} + \epsilon   \nabla \tilde U_t(\theta_t)}^2
  +
  \epsilon^2 \beta^2 \norm{ r_{t - \frac{1}{2}} + r_{t + \frac{1}{2}} }^2 \right) \right)
  \\&\hspace{2em}\cdot
  \min\left(\exp\left( -U(\theta) - \frac{1}{2} \rho_{T - \frac{1}{2}} \right), \exp\left( - U(\theta^*) + \frac{1}{2} \rho_{T - \frac{1}{2}} \right) \right) d\theta_0\cdots d\theta_{T-1}\\
  &=
  \frac{1}{Z} \cdot (1 + \epsilon \beta)^{Td }\cdot \beta^{-\frac{Td}{2}} \cdot 2^{-\frac{3(T-1)d}{2}} \cdot \pi^{-\frac{(T+1)d}{2}} \cdot \epsilon^{-\frac{(3T+2)d}{2}} \cdot \sigma^{(T+1)d}
  \\&\hspace{2em}\cdot
  \mathbf{E}\int\exp\left( -\frac{1}{4 \sigma^2} \left( \norm{ r^*}^2 + \norm{ r }^2 \right) \right)
  \\&\hspace{2em}\cdot
  \exp\left( -\frac{1}{8 \epsilon \beta \sigma^2} \cdot \sum_{t = 0}^{T - 1} \left(
  \norm{r_{t + \frac{1}{2}} - r_{t - \frac{1}{2}} + \epsilon   \nabla \tilde U_t(\theta_t)}^2
  +
  \epsilon^2 \beta^2 \norm{ r_{t - \frac{1}{2}} + r_{t + \frac{1}{2}} }^2 \right) \right)
  \\&\hspace{2em}\cdot
  \exp\left( -\frac{U(\theta) + U(\theta^*)}{2} \right) 
  \cdot \exp\left( - \frac{1}{2} \Abs{ U(\theta) - U(\theta^*) + \rho_{T - \frac{1}{2}} } \right) d\theta_0\cdots d\theta_{T-1}.
\end{align*}
And writing this out explicitly in terms of
\[
  \rho_{T - \frac{1}{2}}
  =
  \frac{1}{2} \epsilon \sigma^{-2}   
  \sum_{t=0}^{T-1} 
  \nabla \tilde U_t(\theta_t)^T \left(r_{t - \frac{1}{2}} + r_{t + \frac{1}{2}} \right),
\]
we get
\begin{align*}
  &\pi(\theta) G(\theta, \theta^*)\\
  &=
  \frac{1}{Z} \cdot (1 + \epsilon \beta)^{Td }\cdot \beta^{-\frac{Td}{2}} \cdot 2^{-\frac{3(T-1)d}{2}} \cdot \pi^{-\frac{(T+1)d}{2}} \cdot \epsilon^{-\frac{(3T+2)d}{2}} \cdot \sigma^{(T+1)d}
  \\&\hspace{2em}\cdot
  \mathbf{E}\int
  \exp\left( -\frac{1}{4 \sigma^2} \left( \norm{ r^* }^2 + \norm{ r }^2 \right) \right)
  \\&\hspace{2em}\cdot
  \exp\left( -\frac{1}{8 \epsilon \beta \sigma^2} \cdot \sum_{t = 0}^{T - 1} \left(
  \norm{r_{t + \frac{1}{2}} - r_{t - \frac{1}{2}} + \epsilon   \nabla \tilde U_t(\theta_t)}^2
  +
  \epsilon^2 \beta^2 \norm{ r_{t - \frac{1}{2}} + r_{t + \frac{1}{2}} }^2 \right) \right)
  \\&\hspace{2em}\cdot
  \exp\left( -\frac{U(\theta) + U(\theta^*)}{2} \right) 
  \\&\hspace{2em}\cdot 
  \exp\left( - \frac{1}{2} \Abs{ U(\theta) - U(\theta^*) + \frac{1}{2} \epsilon \sigma^{-2}   
  \sum_{t=0}^{T-1} 
  \nabla \tilde U_t(\theta_t)^T \left(r_{t - \frac{1}{2}} + r_{t + \frac{1}{2}} \right) } \right) d\theta_0\cdots d\theta_{T-1}.
\end{align*}

Now, for this forward path from $\theta$ to $\theta^*$, consider the reverse leapfrog trajectory from $\theta^*$ to $\theta$.
This trajectory will have the same values for $\theta, \theta_0, \ldots, \theta^*$ in the reversed order and will have negated values for $r_{-\frac{1}{2}}, r_{1 - \frac{1}{2}}, \ldots, r_{T - \frac{1}{2}}$ in the reversed order again.
Because of this negation, the values of $\rho_{\frac{1}{2}}, \rho_{1 + \frac{1}{2}}, \ldots, \rho_{T - \frac{1}{2}}$ will also be negated. It follows that $\pi(\theta^*) G(\theta^*, \theta)$ will have the same expression.

Therefore,
\begin{align*}
\pi(\theta) G(\theta, \theta^*) =\pi(\theta^*) G(\theta^*, \theta).
\end{align*}
This shows that Algorithm~\ref{alg:IMA} with resampling momentum is reversible.

Now we show that the chain satisfies \emph{skew detailed balance} and the stationary distribution of $\theta$ is $\pi(\theta)$ if not resampling momentum. Skew detailed balance means that the chain satisfies the following condition \citep{turitsyn2011irreversible}
\[
\pi(x) G(x, y) = \pi\left(y^{\perp}\right) G\left(y^{\perp}, x^{\perp}\right)
\]
where $G$ is the transition probability. 

By Section~\ref{sec:append:AMA}, we know that a chain that satisfies the above condition will have invariant distribution $\pi(x)$.

In our setting, $x = (\theta, r)$ and $x^{\perp} = (\theta, -r)$. Given this, the skew detailed balance is
\[
\pi(\theta, r) G((\theta, r),(\theta^*, r^*)) = \pi(\theta^*, -r^*) G((\theta^*, -r^*), (\theta, -r)).
\]
Next we will show that Algorithm~\ref{alg:IMA} without resampling momentum satisfies the above condition and it naturally follows that Algorithm~\ref{alg:IMA} without resampling converges to the desired distribution.

We consider the joint distribution of $(\theta, r)$. By a similar analysis of resampling case, we can get that
\begin{align*}
  &\hspace{-2em}\pi(\theta, r) \cdot  G((\theta, r),(\theta^*, r^*)) \\
  &=
  \frac{1}{Z} \cdot (1 + \epsilon \beta)^{Td }\cdot \beta^{-\frac{Td}{2}} \cdot 2^{-\frac{3(T-1)d}{2}} \cdot \pi^{-\frac{(T+1)d}{2}} \cdot \epsilon^{-\frac{(3T+2)d}{2}} \cdot \sigma^{(T+1)d}
  \\&\hspace{2em}\cdot
  \mathbf{E}\int
  \exp\left( -\frac{1}{4 \sigma^2} \left( \norm{ r^* }^2 + \norm{ r }^2 \right) \right)
  \\&\hspace{2em}\cdot
  \exp\left( -\frac{1}{8 \epsilon \beta \sigma^2} \cdot \sum_{t = 0}^{T - 1} \left(
  \norm{r_{t + \frac{1}{2}} - r_{t - \frac{1}{2}} + \epsilon   \nabla \tilde U_t(\theta_t)}^2
  +
  \epsilon^2 \beta^2 \norm{ r_{t - \frac{1}{2}} + r_{t + \frac{1}{2}} }^2 \right) \right)
  \\&\hspace{2em}\cdot
  \exp\left( -\frac{U(\theta) + U(\theta^*)}{2} \right) 
  \\&\hspace{2em}\cdot 
  \exp\left( - \frac{1}{2} \Abs{ U(\theta) - U(\theta^*) + \frac{1}{2} \epsilon \sigma^{-2}
  \sum_{t=0}^{T-1}   
  \nabla \tilde U_t(\theta_t)^T \left(r_{t - \frac{1}{2}} + r_{t + \frac{1}{2}} \right) } \right) d\theta_0\cdots d\theta_{T-1}.
\end{align*}
Again, the reverse trajectory will have the same values for $\theta, \theta_0, \ldots, \theta^*$ and will have negated values for $r_{-\frac{1}{2}}, r_{1 - \frac{1}{2}}, \ldots, r_{T - \frac{1}{2}}$ in the reversed order. Therefore, $\pi(\theta^*, -r^*) G((\theta^*, -r^*), (\theta, -r))$ will have the same expression.

It follows that
\begin{align*}
\pi(\theta, r) G((\theta, r),(\theta^*, r^*)) = \pi(\theta^*, -r^*) G((\theta^*, -r^*), (\theta, -r))
\end{align*}
which is what we want.
\end{proof}

\section{Connection to HMC}\label{sec:relation-hmc}
When using a full-batch, $\beta=0$, and resampling, AMAGOLD becomes HMC. To see this, we first notice that with $\beta = 0$ the update rules of $\theta$ and $r$ are the same as in HMC. The remaining thing is to show $a$ is also the same as in HMC. We rewrite $\rho_{t+\frac{1}{2}}$ as
\begin{align*}
    \rho_{t+\frac{1}{2}} &= \rho_{t-\frac{1}{2}} + \frac{1}{2}\sigma^{-2}\left(r_{t-\frac{1}{2}} - r_{t+\frac{1}{2}}\right)^T\left(r_{t-\frac{1}{2}} + r_{t+\frac{1}{2}}\right)
    = \rho_{t-\frac{1}{2}} + \frac{1}{2}\sigma^{-2}\left(\norm{r_{t-\frac{1}{2}}}^2 - \norm{r_{t+\frac{1}{2}}}^2\right)
\end{align*}
As a result,
\begin{align*}
\rho_{T-\frac{1}{2}} &=
\frac{1}{2}\sigma^{-2}\sum_{t=0}^{T-1} \left(\norm{r_{t-\frac{1}{2}}}^2 - \norm{r_{t+\frac{1}{2}}}^2\right)
= \frac{1}{2}\sigma^{-2}\left(\norm{r_{-\frac{1}{2}}}^2 - \norm{r_{T-\frac{1}{2}}}^2\right)
\end{align*}
It follows that $a$ becomes the same as in HMC.

\section{Proof of Theorem \ref{thm:convergence-rate}}

In this section we prove a bound on the convergence rate of AMAGOLD as compared with second-order Langevin dynamics (L2MC).

\begin{proof}
We start with the expression we derived for the transition probability in the proof of reversibility.
\begin{align*}
  &\pi(\theta) G(\theta, \theta^*)\\
  &=
  \frac{1}{Z} \cdot (1 + \epsilon \beta)^{Td }\cdot \beta^{-\frac{Td}{2}} \cdot 2^{-\frac{(3T-1)d}{2}} \cdot \pi^{-\frac{(T+1)d}{2}} \cdot \epsilon^{-\frac{(3T+2)d}{2}} \cdot \sigma^{(T+1)d}
  \\&\hspace{2em}\cdot
  \mathbf{E}\int
  \exp\left( -\frac{1}{4 \sigma^2} \left( \norm{ r^* }^2 + \norm{ r }^2 \right) \right)
  \\&\hspace{2em}\cdot
  \exp\left( -\frac{1}{8 \epsilon \beta \sigma^2} \cdot \sum_{t = 0}^{T - 1} \left(
  \norm{r_{t + \frac{1}{2}} - r_{t - \frac{1}{2}} + \epsilon \nabla \tilde U_t(\theta_t)}^2
  +
  c^2\epsilon^2 \beta^2 \norm{ r_{t - \frac{1}{2}} + r_{t + \frac{1}{2}} }^2 \right) \right)
  \\&\hspace{2em}\cdot
  \exp\left( -\frac{U(\theta) + U(\theta^*)}{2} \right) 
  \\&\hspace{2em}\cdot 
  \exp\left( - \frac{1}{2} \Abs{ U(\theta) - U(\theta^*) + \frac{1}{2} \epsilon \sigma^{-2}
  \sum_{t=0}^{T-1} 
  \nabla \tilde U_t(\theta_t)^T \left(r_{t - \frac{1}{2}} + r_{t + \frac{1}{2}} \right) } \right)
  d\theta_0\cdots d\theta_{T-1}.
\end{align*}
Since
\[
    \theta_t = \theta_{t-1} + \epsilon \sigma^{-2} r_{t-\frac{1}{2}},
\]
if we define $\theta_{-1}$ and $\theta_{T}$ by convention such that
\[
    \frac{\theta_{-1} + \theta_0}{2} = \theta
    \hspace{2em}\text{and}\hspace{2em}
    \frac{\theta_{T-1} + \theta_T}{2} = \theta^*,
\]
it follows that for all $t \in \{0,\ldots,T-1\}$
\[
    r_{t + \frac{1}{2}} - r_{t - \frac{1}{2}}
    =
    \epsilon^{-1} \sigma^{2} \left( \theta_{t+1} - 2 \theta_{t} + \theta_{t-1} \right)
\]
and
\[
    r_{t + \frac{1}{2}} + r_{t - \frac{1}{2}}
    \epsilon^{-1} \sigma^{2} \left( \theta_{t+1} - \theta_{t-1} \right)
\]
so we can write the above transition probability explicitly in terms of the $\theta_t$ as
\begin{align*}
  &\pi(\theta) G(\theta, \theta^*)\\
  &=
  \frac{1}{Z} \cdot (1 + \epsilon \beta)^{Td }\cdot \beta^{-\frac{Td}{2}} \cdot 2^{-\frac{(3T-1)d}{2}} \cdot \pi^{-\frac{(T+1)d}{2}} \cdot \epsilon^{-\frac{(3T+2)d}{2}} \cdot \sigma^{(T+1)d}
  \\&\hspace{2em}\cdot
  \mathbf{E}\int
  \exp\left( -\frac{c}{4 \sigma^2} \left( \norm{ r^* }^2 + \norm{ r }^2 \right) \right)
  \\&\hspace{2em}\cdot
  \exp\Bigg( -\frac{1}{8 \epsilon \beta \sigma^2} \cdot \sum_{t = 0}^{T - 1} \Bigg(
  \norm{\epsilon^{-1} \sigma^{2} \left( \theta_{t+1} - 2 \theta_{t} + \theta_{t-1} \right) + \epsilon \nabla \tilde U_t(\theta_t) }^2
  \\&\hspace{2em}+
  c^2\epsilon^2 \beta^2 \norm{ \epsilon^{-1} \sigma^{2} \left( \theta_{t+1} - \theta_{t-1} \right) }^2 \Bigg) \Bigg)
  \cdot
  \exp\left( -\frac{U(\theta) + U(\theta^*)}{2} \right) 
  \\&\hspace{2em}\cdot 
  \exp\left( - \frac{1}{2} \Abs{ U(\theta) - U(\theta^*) + \frac{1}{2} \epsilon \sigma^{-2}
  \sum_{t=0}^{T-1} \nabla \tilde U(\theta_t)^T \left(\epsilon^{-1} \sigma^{2} \left( \theta_{t+1} - \theta_{t-1} \right) \right) } \right)
  \;\\&\hspace{2em} \cdot d\theta_0\cdots d\theta_{T-1}.
\end{align*}

Simplifying this a bit, we get
\begin{align*}
  &\pi(\theta) G(\theta, \theta^*)\\
  &=
  \frac{1}{Z} \cdot (1 + \epsilon \beta)^{Td }\cdot \beta^{-\frac{Td}{2}} \cdot 2^{-\frac{(3T-1)d}{2}} \cdot \pi^{-\frac{(T+1)d}{2}} \cdot \epsilon^{-\frac{(3T+2)d}{2}} \cdot \sigma^{(T+1)d}
  \\&\hspace{2em} \cdot \mathbf{E}\int
  \exp\left( -\frac{\sigma^2}{\epsilon^2} \left( \norm{ \theta^* - \theta_{T-1} }^2 + \norm{ \theta_0 - \theta }^2 \right) \right)
  \\&\hspace{2em}\cdot
  \exp\left( -\frac{\sigma^2}{8 \epsilon^3 \beta} \cdot \sum_{t = 0}^{T - 1}
  \norm{ \theta_{t+1} - 2 \theta_{t} + \theta_{t-1} + \epsilon^2 \sigma^{-2} \nabla \tilde U_t(\theta_t) }^2 \right)
  \\&\hspace{2em}\cdot
  \exp\left( -\frac{\beta \sigma^2}{8 \epsilon} \cdot \sum_{t = 0}^{T - 1}
  \norm{ \theta_{t+1} - \theta_{t-1} }^2 \right)
  \\&\hspace{2em}\cdot
  \exp\left( -\frac{U(\theta) + U(\theta^*)}{2} \right) 
  \\&\hspace{2em}\cdot 
  \exp\left( - \frac{1}{2} \Abs{ U(\theta) - U(\theta^*) + \frac{1}{2}
  \sum_{t=0}^{T-1} 
  \nabla \tilde U_t(\theta_t)^T \left( \theta_{t+1} - \theta_{t-1} \right) } \right)
  \; d\theta_0\cdots d\theta_{T-1}.
\end{align*}
Next, let
\begin{align*}
    N_t &= \nabla \tilde U_t(\theta_t) - \nabla U_t(\theta_t), \\
    A_t &= \theta_{t+1} - 2 \theta_{t} + \theta_{t-1} + \epsilon^2 \sigma^{-2} \nabla U_t(\theta_t), \\
    B_t &= \theta_{t+1} - \theta_{t-1}, \\
    C_t &=
    U(\theta) - U(\theta^*) + \frac{1}{2}
      \sum_{t=0}^{T-1} 
      \nabla U_t(\theta_t)^T \left( \theta_{t+1} - \theta_{t-1} \right).
\end{align*}
Notice that only $N_t$ depends on the randomness of the stochastic gradient samples. 
Then,
\begin{align*}
  &\pi(\theta) G(\theta, \theta^*)\\
  &=
  \frac{1}{Z} \cdot (1 + \epsilon \beta)^{Td }\cdot \beta^{-\frac{Td}{2}} \cdot 2^{-\frac{(3T-1)d}{2}} \cdot \pi^{-\frac{(T+1)d}{2}} \cdot \epsilon^{-\frac{(3T+2)d}{2}} \cdot \sigma^{(T+1)d}
  \\&\hspace{2em}\cdot \int \mathbf{E}\Bigg[
  \exp\left( -\frac{\sigma^2}{\epsilon^2} \left( \norm{ \theta^* - \theta_{T-1} }^2 + \norm{ \theta_0 - \theta }^2 \right) \right)
  \\&\hspace{2em}\cdot
  \exp\left( -\frac{\sigma^2}{8 \epsilon^3 \beta} \cdot \sum_{t = 0}^{T - 1}
  \left(
    \norm{ A_t  }^2
    +
    2 \epsilon^2 \sigma^{-2} A_t^T N_t
    +
    \epsilon^4 \sigma^{-4} \norm{ N_t }^2
\right) \right)
  \\&\hspace{2em}\cdot
  \exp\left( -\frac{\beta \sigma^2}{8 \epsilon} \cdot \sum_{t = 0}^{T - 1}
  \norm{  B_t }^2 \right)
  \\&\hspace{2em}\cdot
  \exp\left( -\frac{U(\theta) + U(\theta^*)}{2} \right) 
  \\&\hspace{2em}\cdot 
  \exp\left( - \frac{1}{2} \Abs{ C_t + \frac{1}{2}
  \sum_{t=0}^{T-1} 
  N_t^T B_t } \right) \Bigg]
  \; d\theta_0\cdots d\theta_{T-1}.
\end{align*}
Now, for any constant $c > 1$, we can bound
\begin{align*}
  \Exv{ \Abs{ \sum_{t=0}^{T-1} N_t^T B_t } }
  &\le
  \sqrt{ \Exv{ \left( \sum_{t=0}^{T-1} N_t^T B_t \right)^2 } }
  \\&=
  \sqrt{ \sum_{t=0}^{T-1} B_t^T \Exv{ N_t N_t^T } B_t } 
  \\&\le
  \sqrt{ \sum_{t=0}^{T-1} \frac{V^2}{d} \norm{B_t}^2 } 
  \\&\le
  \frac{V^2}{d} \frac{\epsilon}{2 (c -1) \beta \sigma^2}
  +
  (c - 1) \frac{\beta \sigma^2}{2 \epsilon} \sum_{t=0}^{T-1} \norm{B_t}^2.
\end{align*}
Additionally, we know that $\Exv{N_t} = 0$ and
\[
  \Exv{\norm{ N_t }^2} = \Exv{ \trace{N_t N_t^T} } \le \trace{\frac{V^2}{d} I} = V^2.
\]
So, since by Jensen's inequality, $\Exv{\exp(X)} \ge \exp(\Exv{X})$, we can bound this with
\begin{align*}
  &\pi(\theta) G(\theta, \theta^*)\\
  &\ge
  \frac{1}{Z} \cdot (1 + \epsilon \beta)^{Td }\cdot \beta^{-\frac{Td}{2}} \cdot 2^{-\frac{(3T-1)d}{2}} \cdot \pi^{-\frac{(T+1)d}{2}} \cdot \epsilon^{-\frac{(3T+2)d}{2}} \cdot \sigma^{(T+1)d}
  \\&\hspace{2em}\cdot \int
  \exp\left( -\frac{\sigma^2}{\epsilon^2} \left( \norm{ \theta^* - \theta_{T-1} }^2 + \norm{ \theta_0 - \theta }^2 \right) \right)
  \\&\hspace{2em}\cdot
  \exp\left( -\frac{\sigma^2}{8 \epsilon^3 \beta} \cdot \sum_{t = 0}^{T - 1}
    \norm{ A_t  }^2
  \right)
  \cdot
  \exp\left( -\frac{\epsilon T V^2}{8 \sigma^2 \beta} \right)
  \\&\hspace{2em}\cdot
  \exp\left( -\frac{c \sigma^2 \beta}{8 \epsilon} \cdot \sum_{t = 0}^{T - 1}
  \norm{  B_t }^2 \right)
  \\&\hspace{2em}\cdot
  \exp\left( -\frac{U(\theta) + U(\theta^*)}{2} \right) 
  \\&\hspace{2em}\cdot 
  \exp\left( - \frac{1}{2} \Abs{ C_t } \right)
  \cdot
  \exp\left( - \frac{1}{(c-1) T d}
  \cdot
  \frac{\epsilon T V^2}{8 \sigma^2 \beta} \right)
  \; d\theta_0\cdots d\theta_{T-1}.
\end{align*}
Now, this is a lower bound on the AMAGOLD chain with parameters $(\epsilon, \sigma, \beta)$.
Next, we consider the transition probability of a \emph{rescaled} chain, with slightly different parameters, that will be set as a function of $c$.
Specifically, consider the chain with parameters $(\epsilon, \sigma \cdot c^{-1/4}, \beta \cdot c^{-1/2})$.
(We will set the parameter $c$ later; at this point in the proof it is just an arbitrary constant $c > 1$.)
If we call this rescaled chain $G_r$, then by substitution of the parameters into the above expression, we get
\begin{align*}
  &\pi(\theta) G_r(\theta, \theta^*)\\
  &\ge
  \frac{1}{Z} \cdot (1 + c^{-1/2} \epsilon \beta)^{Td} \cdot
  \beta^{-\frac{Td}{2}} \cdot
  2^{-\frac{3(T-1)d}{2}} \cdot \pi^{-\frac{(T+1)d}{2}} \cdot c^{-\frac{d}{4}} \cdot \epsilon^{-\frac{(3T+2)d}{2}} \cdot \sigma^{(T+1)d}
  \\&\hspace{2em}\cdot
  \int \mathbf{E}
  \exp\left( -\frac{\sigma^2}{c^{1/2} \epsilon^2} \left( \norm{ \theta^* - \theta_{T-1} }^2 + \norm{ \theta_0 - \theta }^2 \right) \right)
  \\&\hspace{2em}\cdot
  \exp\left( -\frac{\sigma^2}{8 \epsilon^3 \beta} \cdot \sum_{t = 0}^{T - 1}
    \norm{ A_t  }^2
  \right)
  \cdot
  \exp\left( -\frac{c \epsilon T V^2}{8 \sigma^2 \beta} \right)
  \\&\hspace{2em}\cdot
  \exp\left( -\frac{\sigma^2 \beta}{8 \epsilon} \cdot \sum_{t = 0}^{T - 1}
  \norm{  B_t }^2 \right)
  \\&\hspace{2em}\cdot
  \exp\left( -\frac{U(\theta) + U(\theta^*)}{2} \right) 
  \\&\hspace{2em}\cdot 
  \exp\left( - \frac{1}{2} \Abs{ C_t } \right)
  \cdot
  \exp\left( - \frac{1}{(c-1) T d}
  \cdot
  \frac{c \epsilon T V^2}{8 \sigma^2 \beta} \right)
  \; d\theta_0\cdots d\theta_{T-1}\\
  &\ge
  \frac{1}{Z} \cdot (1 + c^{-1/2} \epsilon \beta)^{Td} \cdot
  \beta^{-\frac{Td}{2}} \cdot
  2^{-\frac{3(T-1)d}{2}} \cdot \pi^{-\frac{(T+1)d}{2}} \cdot c^{-\frac{d}{4}} \cdot \epsilon^{-\frac{(3T+2)d}{2}} \cdot \sigma^{(T+1)d}
  \\&\hspace{2em}
  \cdot
  \exp\left( -\frac{c \epsilon T V^2}{8 \sigma^2 \beta} \right)
  \cdot
  \exp\left( - \frac{1}{(c-1) T d}
  \cdot
  \frac{c \epsilon T V^2}{8 \sigma^2 \beta} \right)
  \\&\hspace{2em} \cdot
  \int \mathbf{E}
  \exp\left( -\frac{\sigma^2}{\epsilon^2} \left( \norm{ \theta^* - \theta_{T-1} }^2 + \norm{ \theta_0 - \theta }^2 \right) \right)
  \\&\hspace{2em}\cdot
  \exp\left( -\frac{\sigma^2}{8 \epsilon^3 \beta} \cdot \sum_{t = 0}^{T - 1}
    \norm{ A_t  }^2
  \right)
  \\&\hspace{2em}\cdot
  \exp\left( -\frac{\sigma^2 \beta}{8 \epsilon} \cdot \sum_{t = 0}^{T - 1}
  \norm{  B_t }^2 \right)
  \\&\hspace{2em}\cdot
  \exp\left( -\frac{U(\theta) + U(\theta^*)}{2} \right) 
  \\&\hspace{2em}\cdot 
  \exp\left( - \frac{1}{2} \Abs{ C_t } \right)
  \; d\theta_0\cdots d\theta_{T-1}.
\end{align*}
On the other hand, consider the transition probability of the full-gradient L2MC chain with parameters $(\epsilon, \sigma, \beta)$.
This chain will be the same as the AMAGOLD chain, except that $N_t = 0$ always.
So, if we call this chain's transition probability $\bar G$, we will have
\begin{align*}
  &\pi(\theta) \bar G(\theta, \theta^*)\\
  &=
  \frac{1}{Z} \cdot (1 + \epsilon \beta)^{Td }\cdot \beta^{-\frac{Td}{2}} \cdot 2^{-\frac{(3T-1)d}{2}} \cdot \pi^{-\frac{(T+1)d}{2}} \cdot \epsilon^{-\frac{(3T+2)d}{2}} \cdot \sigma^{(T+1)d}
  \\&\hspace{2em}\cdot \int
  \exp\left( -\frac{\sigma^2}{\epsilon^2} \left( \norm{ \theta^* - \theta_{T-1} }^2 + \norm{ \theta_0 - \theta }^2 \right) \right)
  \\&\hspace{2em}\cdot
  \exp\left( -\frac{\sigma^2}{8 \epsilon^3 \beta} \cdot \sum_{t = 0}^{T - 1}
    \norm{ A_t  }^2  \right)
  \\&\hspace{2em}\cdot
  \exp\left( -\frac{\beta \sigma^2}{8 \epsilon} \cdot \sum_{t = 0}^{T - 1}
  \norm{  B_t }^2 \right)
  \\&\hspace{2em}\cdot
  \exp\left( -\frac{U(\theta) + U(\theta^*)}{2} \right) 
  \\&\hspace{2em}\cdot 
  \exp\left( - \frac{1}{2} \Abs{ C_t } \right)
  \; d\theta_0\cdots d\theta_{T-1}.
\end{align*}
Using this, we can simplify our bound on the transition probability of the AMAGOLD chain to
\begin{align*}
  \pi(\theta) G_r(\theta, \theta^*)
  &\ge
  \left( \frac{ 1 + c^{-1/2} \epsilon \beta }{1 + \epsilon \beta} \right)^{Td} \cdot
  c^{-\frac{d}{4}}
  \cdot
  \exp\left( -\frac{c \epsilon T V^2}{8 \sigma^2 \beta} \right)
  \cdot
  \exp\left( - \frac{1}{(c-1) T d}
  \cdot
  \frac{c \epsilon T V^2}{8 \sigma^2 \beta} \right)
  \\&\hspace{2em}\cdot
  \pi(\theta) \bar G(\theta, \theta^*).
\end{align*}
Thus,
\begin{align*}
  \frac{\pi(\theta) G_r(\theta, \theta^*)}{\pi(\theta) \bar G(\theta, \theta^*)}
  \ge
  \left( \frac{ 1 + c^{-1/2} \epsilon \beta }{1 + \epsilon \beta} \right)^{Td} \cdot
  c^{-\frac{d}{4}}
  \cdot
  \exp\left( -\frac{c \epsilon T V^2}{8 \sigma^2 \beta} \right)
  \cdot
  \exp\left( - \frac{1}{(c-1) T d}
  \cdot
  \frac{c \epsilon T V^2}{8 \sigma^2 \beta} \right).
\end{align*}
All that remains to get a bound is to set $c$ appropriately.
Since
\[
  c^{-\frac{d}{4}}
  =
  \exp\left(-\frac{d}{4} \log(c) \right)
  \ge
  \exp\left(-\frac{d}{4} (c - 1) \right),
\]
and
\[
  c^{-1/2}
  \ge
  1 - \frac{c - 1}{2},
\]
we can bound this with
\begin{align*}
  \frac{\pi(\theta) G_r(\theta, \theta^*)}{\pi(\theta) \bar G(\theta, \theta^*)}
  &\ge
  \left( \frac{ 1 + \left( 1 - \frac{c - 1}{2} \right) \epsilon \beta }{1 + \epsilon \beta} \right)^{Td}
  \cdot
  \exp\left(-\frac{d}{4} (c - 1)
   -\frac{c \epsilon T V^2}{8 \sigma^2 \beta}
   - \frac{1}{(c-1) T d}
  \cdot
  \frac{c \epsilon T V^2}{8 \sigma^2 \beta} \right) \\
  &\ge
  \left( 1 - \frac{ (c - 1) \epsilon \beta }{2 (1 + \epsilon \beta)} \right)^{Td}
  \cdot
  \exp\left(-\frac{d}{4} (c - 1)
   -\frac{c \epsilon T V^2}{8 \sigma^2 \beta}
   - \frac{1}{(c-1) T d}
  \cdot
  \frac{c \epsilon T V^2}{8 \sigma^2 \beta} \right).
\end{align*}
Since for any $0 \le x < 1/2$, it holds that $1 - x \ge \exp(-2x)$,
as long as
\[
  \frac{ (c - 1) \epsilon \beta }{1 + \epsilon \beta}
  \le
  1,
\]
it holds that
\[
  1 - \frac{ (c - 1) \epsilon \beta }{2 (1 + \epsilon \beta)}
  \ge
  \exp\left(-\frac{ (c - 1) \epsilon \beta }{1 + \epsilon \beta} \right).
\]
So, under this assumption,
\begin{align*}
  \frac{\pi(\theta) G_r(\theta, \theta^*)}{\pi(\theta) \bar G(\theta, \theta^*)}
  &\ge
  \exp\left( - \frac{ (c - 1) \epsilon \beta Td }{1 + \epsilon \beta} 
   -\frac{d}{4} (c - 1)
   -\frac{c \epsilon T V^2}{8 \sigma^2 \beta}
   - \frac{1}{(c-1) T d}
  \cdot
  \frac{c \epsilon T V^2}{8 \sigma^2 \beta} \right) \\
  &=
  \exp\left( -\frac{\epsilon T V^2}{8 \sigma^2 \beta} - \frac{\epsilon V^2}{8 \sigma^2 \beta d} \right)
  \\&\hspace{2em}\cdot
  \exp\left( - (c - 1) \left(
    \frac{ (1 + \epsilon \beta (1 + 4 T)) d }{4 (1 + \epsilon \beta)} 
    +
    \frac{\epsilon T V^2}{8 \sigma^2 \beta}
   \right)
   -
   \frac{\epsilon V^2}{8 (c - 1) \sigma^2 \beta d} \right) \\
  &=
  \exp\left( -\frac{\epsilon T V^2}{8 \sigma^2 \beta} - \frac{\epsilon V^2}{8 \sigma^2 \beta d} \right)
  \\&\hspace{2em}\cdot
  \exp\left( - (c - 1) \left(
    \frac{3}{2} Td
    +
    \frac{\epsilon T V^2}{8 \sigma^2 \beta}
   \right)
   -
   \frac{\epsilon V^2}{8 (c - 1) \sigma^2 \beta d} \right).
\end{align*}
If we also assume that 
\[
  \frac{\epsilon V^2}{4 \sigma^2 \beta d} \le 1,
\]
then
\[
  \frac{\epsilon T V^2}{8 \sigma^2 \beta} \le \frac{Td}{2},
\]
and so
\begin{align*}
  \frac{\pi(\theta) G_r(\theta, \theta^*)}{\pi(\theta) \bar G(\theta, \theta^*)}
  &\ge
  \exp\left( -\frac{\epsilon T V^2}{8 \sigma^2 \beta} - \frac{\epsilon V^2}{8 \sigma^2 \beta d} \right)
  \\&\hspace{2em}\cdot
  \exp\left( - (c - 1) 2 Td
   -
   \frac{\epsilon V^2}{8 (c - 1) \sigma^2 \beta d} \right).
\end{align*}
Next, set
\[
  c - 1
  =
  \sqrt{
    \frac{\epsilon V^2}{16 \sigma^2 \beta T d^2}
  }.
\]
From this, we will get
\begin{align*}
  \frac{\pi(\theta) G_r(\theta, \theta^*)}{\pi(\theta) \bar G(\theta, \theta^*)}
  &\ge
  \exp\left( -\frac{\epsilon T V^2}{8 \sigma^2 \beta} - \frac{\epsilon V^2}{8 \sigma^2 \beta d} \right)
  \cdot
  \exp\left(
   -
   \sqrt{
      \frac{\epsilon T V^2}{\sigma^2 \beta}
    }
  \right) \\
  &\ge
  \exp\left( -\frac{\epsilon T V^2}{4 \sigma^2 \beta} - 
   \sqrt{
      \frac{\epsilon T V^2}{\sigma^2 \beta}
    }
  \right).
\end{align*}
Now, in order for this to hold, we needed
\[
  \frac{ (c - 1) \epsilon \beta }{1 + \epsilon \beta}
  \le
  1.
\]
With our setting of $c$, and our other assumption,
\[
  c - 1
  =
  \sqrt{
    \frac{\epsilon V^2}{16 \sigma^2 \beta T d^2}
  }
  =
  \sqrt{
    \frac{\epsilon V^2}{4 \sigma^2 \beta d}
    \cdot
    \frac{1}{4 T d}
  }
  \le
  \sqrt{
    \frac{1}{4 T d}
  }
  \le
  1,
\]
so the bound will trivially hold.
Thus the only added assumption we needed is the one stated in the Theorem statement, that
\[
  \frac{\epsilon V^2}{4 \sigma^2 \beta d} \le 1.
\]
Now we apply the standard Dirichlet form argument. The spectral gap of a Markov chain can be written as \citep{aida1998uniform}
\[
\gamma = \inf_{f\in L^2_0(\pi): Var_{\pi}[f] = 1} \mathcal{E}(f)
\]
where $L^2_0(\pi)$ denotes the Hilbert space of all functions that are square integrable with respect to probability measure $\pi$ and have mean zero. $\mathcal{E}(f)$ is the Dirichlet form of a Markov chain associated with transition operator $T$ \citep{fukushima2010dirichlet}:
\begin{align*}
\mathcal{E}(f) 
	= 
	\frac{1}{2}\int\int\left[\left(f(\theta)-f(\theta^*)\right)^2\right]G(\theta,\theta^*)\pi(\theta)d\theta d\theta^*
\end{align*}

By the expression of the spectral gap, it follows that
\begin{align*}
\gamma &= \inf_{f\in L^2_0(\pi): Var_{\pi}[f] = 1} \left[\frac{1}{2}\int\int\left[\left(f(\theta)-f(\theta^*)\right)^2\right]G(\theta,\theta^*)\pi(\theta)d\theta d\theta^*\right]\\
&\geq \exp\left( -\frac{\epsilon T V^2}{4 \sigma^2 \beta} - 
   \sqrt{
      \frac{\epsilon T V^2}{\sigma^2 \beta}
    }
  \right)
\cdot \inf_{f\in L^2_0(\pi): Var_{\pi}[f] = 1} \left[\frac{1}{2}\int\int\left[\left(f(\theta)-f(\theta^*)\right)^2\right]\bar G(\theta,\theta)\pi(\theta)d\theta d\theta^*\right]\\
&= \exp\left( -\frac{\epsilon T V^2}{4 \sigma^2 \beta} - 
   \sqrt{
      \frac{\epsilon T V^2}{\sigma^2 \beta}
    }
  \right) \cdot \bar\gamma
\end{align*}
This finishes the proof.
\end{proof}

\section{Reformulation of AMAGOLD Algorithm}\label{sec:append:reform}
We reformulate our algorithm by setting $v=\epsilon\sigma^{-2}r$,$b=\epsilon\beta$, $h=\epsilon^2\sigma^{-2}$ and outline the algorithm after reformulation in Algorithm \ref{alg:re-IMA}.

\begin{algorithm}[H]
\caption{Reformulated AMAGOLD}
\label{alg:re-IMA}
\begin{algorithmic}[1]
  \STATE \textbf{given:} Energy $U$, initial state $\theta \in \Theta$
  \LOOP
    \STATE \textbf{optionally, resample momentum: } $v \sim \mathcal{N}(0, h{\bf I})$
    \STATE \textbf{initialize momentum and energy acc: } $v_{-\frac{1}{2}} \leftarrow v$, $\rho_{- \frac{1}{2}} \leftarrow 0$
    \STATE \textbf{half position update:} $\theta_{0} \leftarrow \theta + \frac{1}{2} v_{-\frac{1}{2}}$
    \FOR {$t = 0$ \textbf{to} $T-1$}
      \IF {$t\neq 0$ }
        \STATE \textbf{position update:} $\theta_{t} \leftarrow \theta_{t-1} + v_{t-\frac{1}{2}}$
    \ENDIF
    \STATE \textbf{sample noise } $\eta_t \sim \mathcal{N}(0, 4 h b)$
    \STATE \textbf{sample random energy component} $\tilde U_t$
    \STATE \textbf{update momentum: } $v_{t + \frac{1}{2}} \leftarrow \left((1-b)v_{t - \frac{1}{2}} - h \nabla \tilde U_t(\theta_t) + \eta_t\right)/(1+b)$
    \STATE \textbf{update energy acc: $\rho_{t + \frac{1}{2}} \leftarrow \rho_{t - \frac{1}{2}} + \frac{1}{2} \nabla \tilde U_t(\theta_t)^T \left(v_{t - \frac{1}{2}} + v_{t + \frac{1}{2}}\right)$}
    \ENDFOR
    \STATE \textbf{half position update:} $\theta_T \leftarrow \theta_{T-1} + \frac{1}{2} v_{T-\frac{1}{2}}$
    \STATE \textbf{new values:} $\theta^* \leftarrow \theta_T$, $v^* \leftarrow v_{T-\frac{1}{2}}$
    \STATE $a \leftarrow \exp\left( U(\theta) - U(\theta^*) + \rho_{T - \frac{1}{2}} \right)$
    \STATE \textbf{with probability } $\min(1, a)$ \textbf{ update } $\theta \leftarrow \theta^*$, $v\leftarrow v^*$ (as long as $\theta^* \in \Theta$)
    \STATE \textbf{otherwise update } $v \leftarrow - v_{-\frac{1}{2}}$
  \ENDLOOP
\end{algorithmic}
\end{algorithm}
\vspace{10cm}
\section{Additional Experiments Results and Setting Details}\label{sec:append:exp2}
\subsection{Double Well Potential}\label{sec:append:dwp}
We visualize the estimated density on additional step size settings. Consistent with Figure \ref{fig:1dim}d, it is clear here that SGHMC is very sensitive to step size. A small change in step size will cause a big difference in the estimated density. In contrast, AMAGOLD is more robust and can work well with a large range of step sizes. 

When the setup of step size is inappropriate, as in Figures \ref{fig:1dim_2}a and b where it is fixed to be too small, either SGHMC or AMAGOLD converges in the training time. This is because the chain moves too slowly toward the stationary distribution. However, AMAGOLD with step size tuning is able to automatically adjust the step size based on the information provided by M-H step. As shown in Figure \ref{fig:1dim}c, tuned AMAGOLD can determine a step size that causes convergence given the same training time budget. All results are obtained by collecting $10^5$ samples with 1000 burn-in samples.

\begin{figure}[h!]
\centering
	\begin{tabular}{cc}		
		\includegraphics[width=6cm]{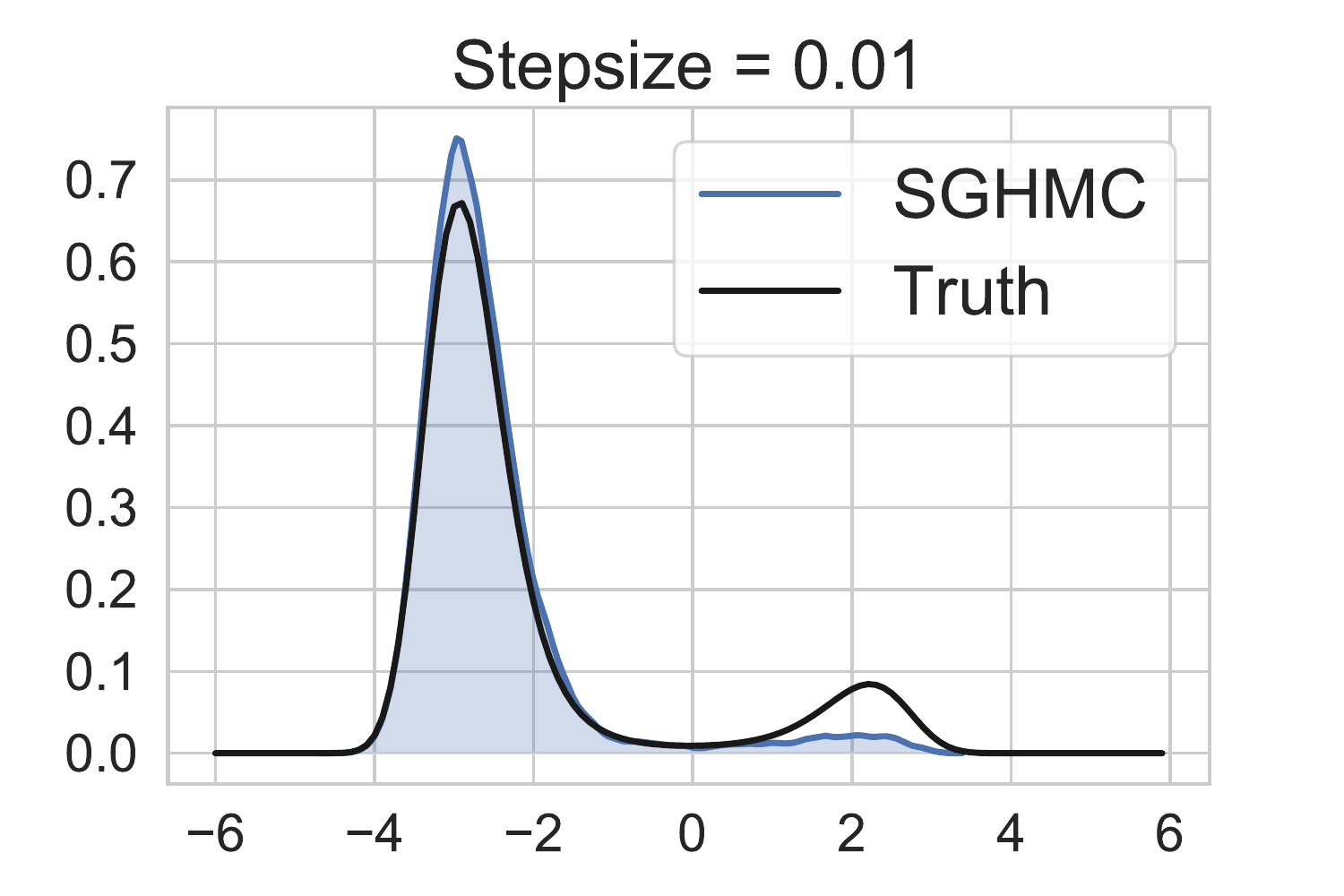}  &
        \includegraphics[width=6cm]{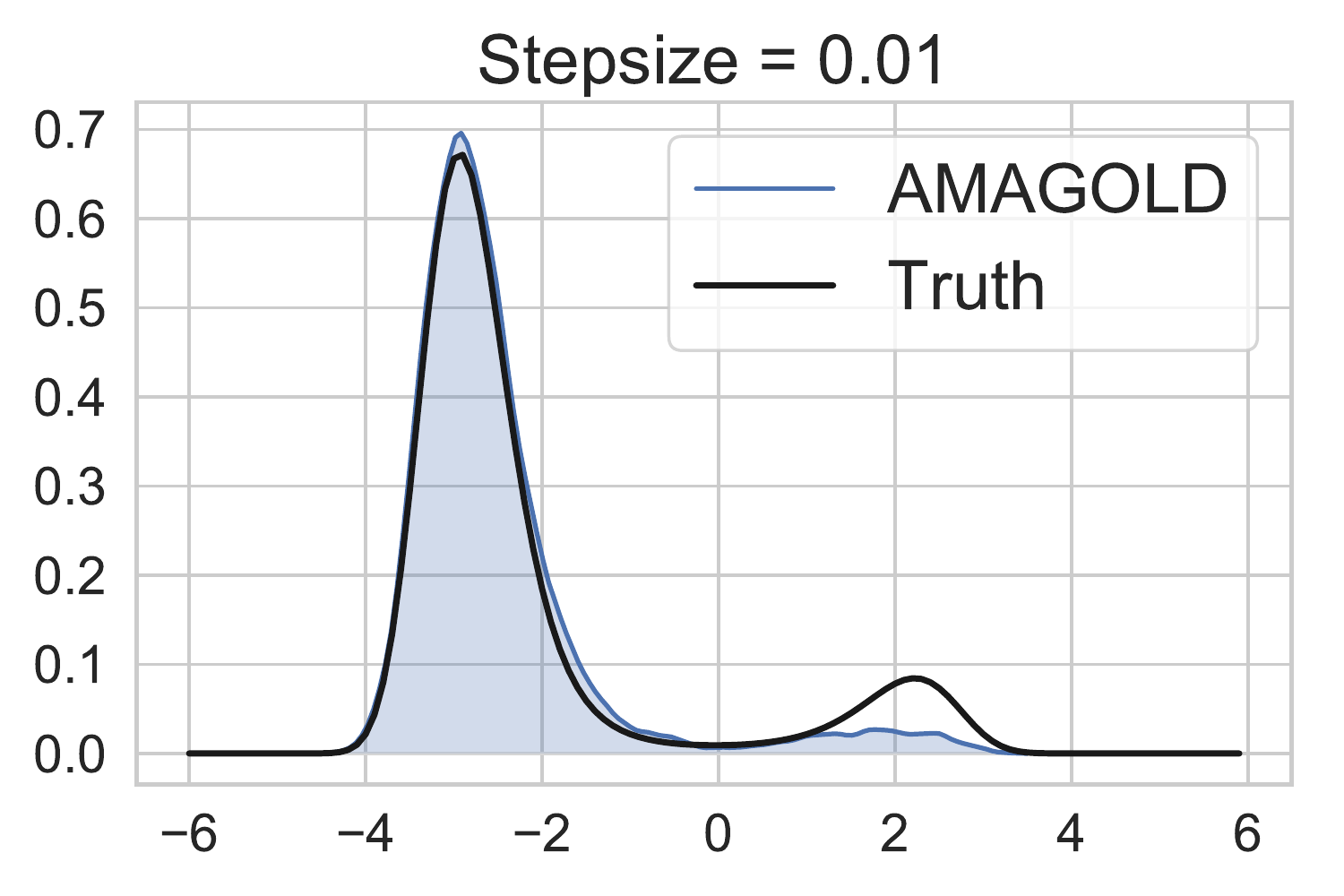}  
		\\
		(a)&
		(b)
		\\
		\includegraphics[width=6cm]{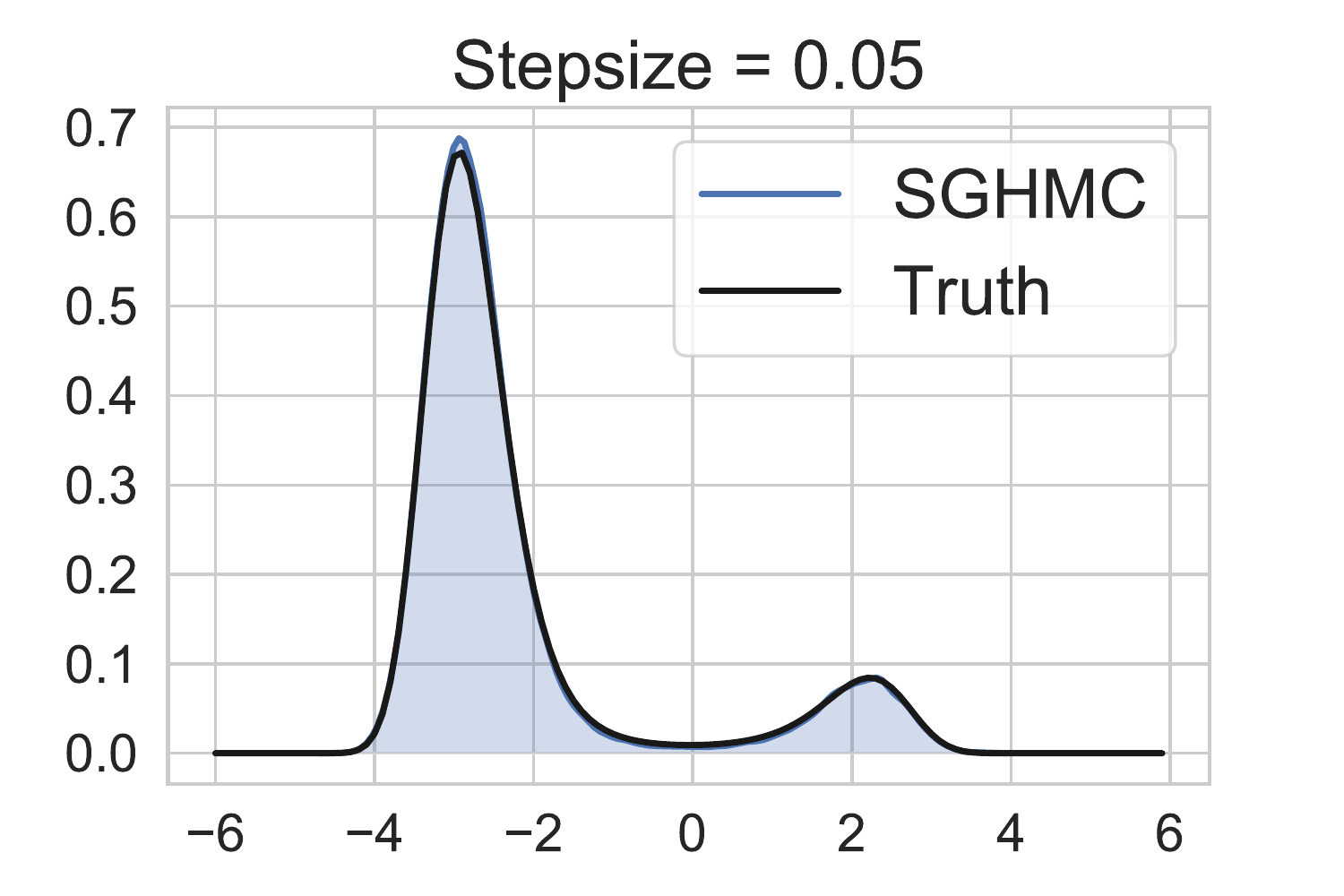}  &
		\includegraphics[width=6cm]{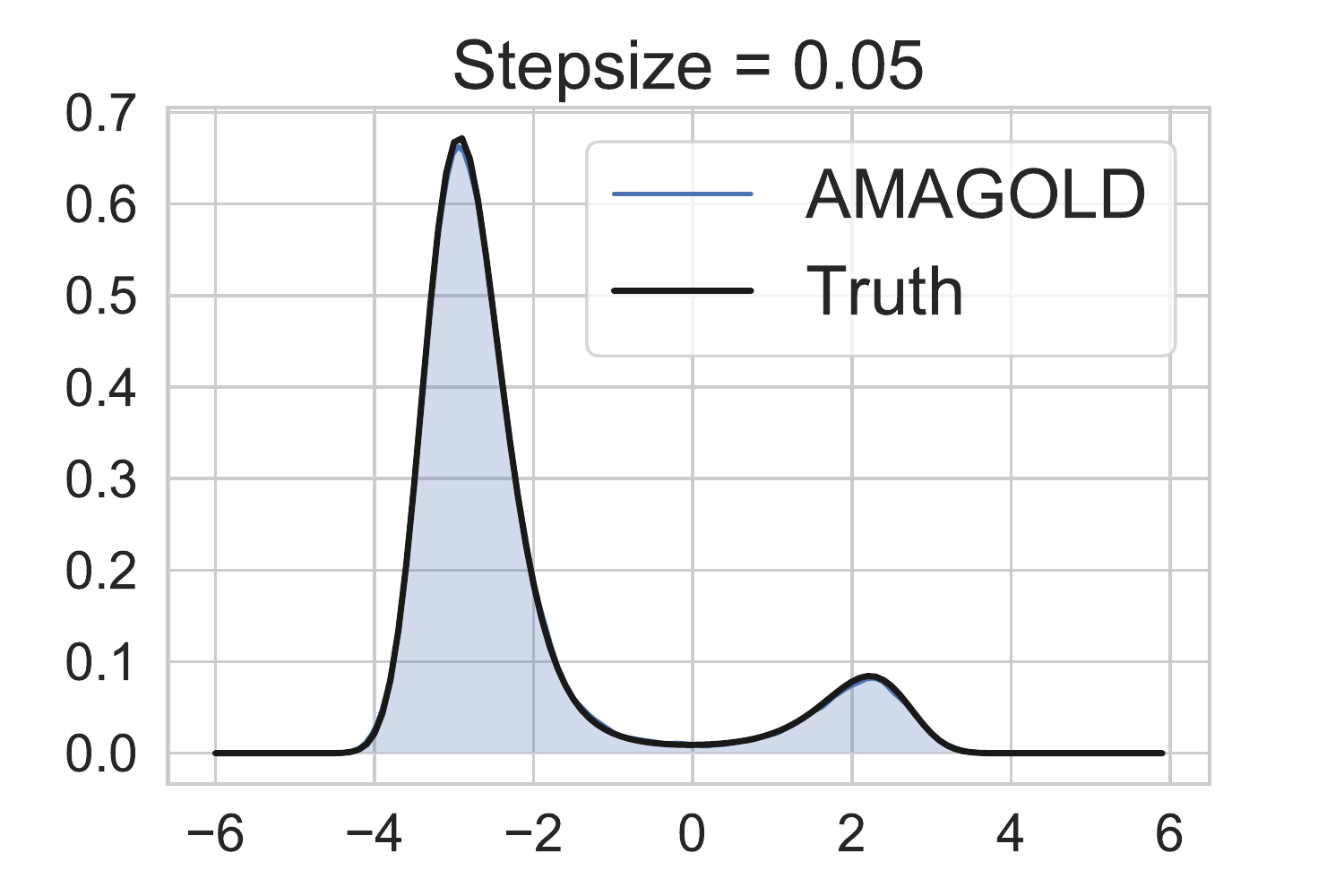}
		\\
		(c)&
		(d)
		\\
		\includegraphics[width=6cm]{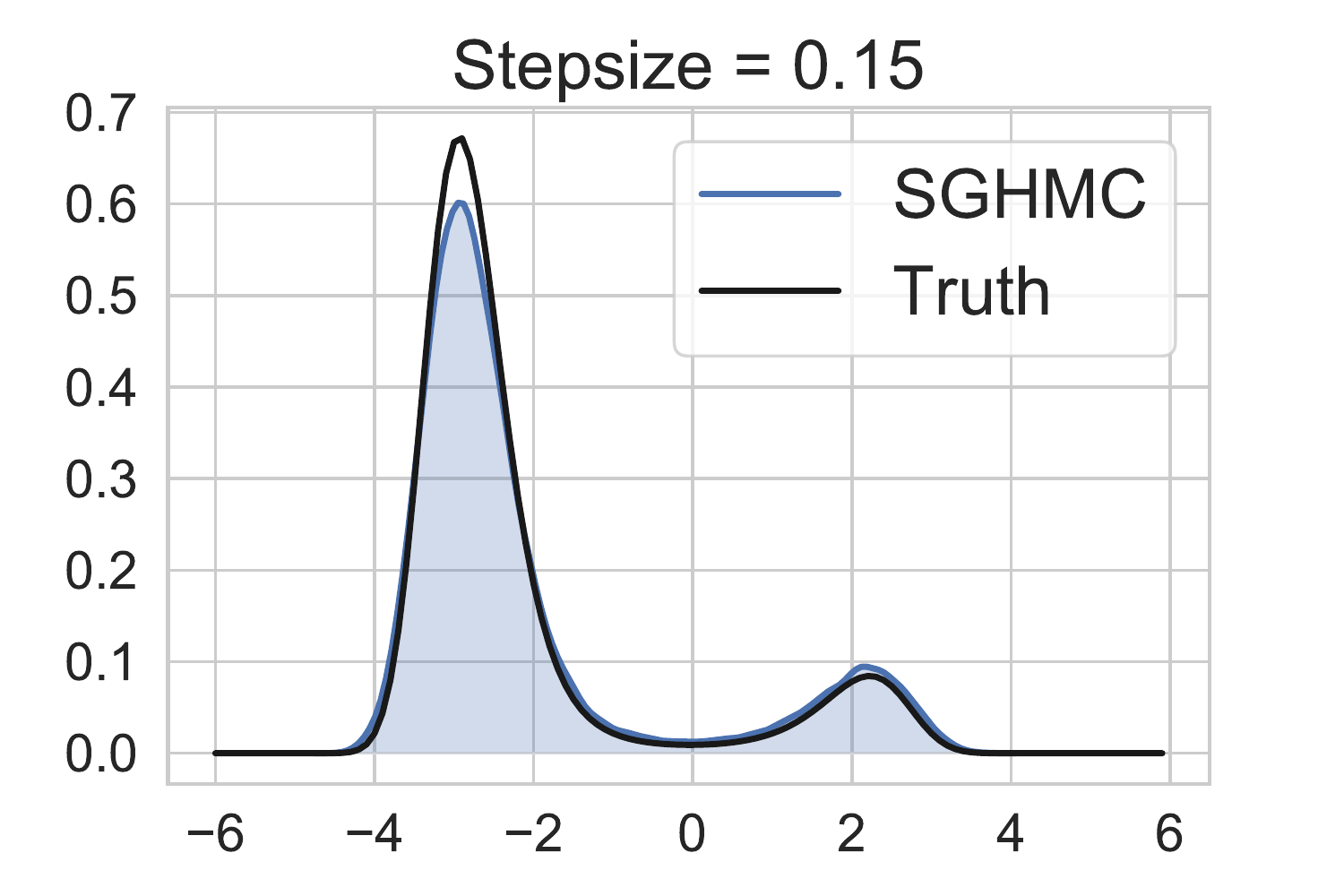} &
		\includegraphics[width=6cm]{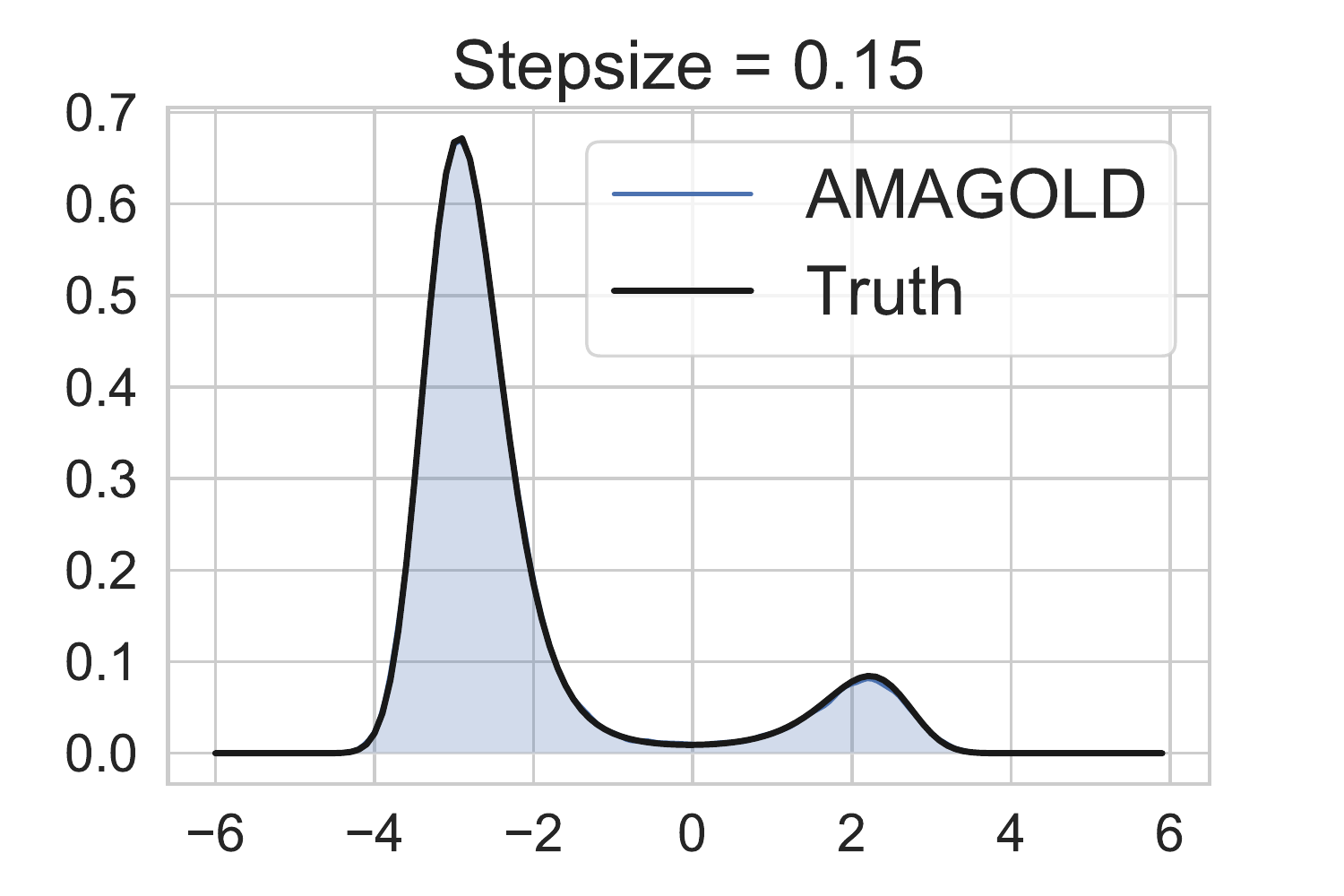}
		\\
		(e)&
		(f)
	\end{tabular}
	\caption{Estimated densities of SGHMC (1st column) and AMAGOLD (2nd column) on varying step sizes.}
	\label{fig:1dim_2}
\end{figure} 
\subsection{Two-Dimensional Synthetic Distributions}
\subsubsection{Analytical Expression}\label{sec:append:dims}
\begin{align*}
    &\text{Dist1: }\mathcal{N}(z_1; z_2^2/4,1)\mathcal{N}(z_2;0,4)\\
    &\text{Dist2: }0.5\mathcal{N}\Bigg(\bm z; 0 ,\begin{bmatrix}2&1.8\\1.8&2 \end{bmatrix}\Bigg) + 0.5\mathcal{N}\Bigg(\bm z; 0 ,\begin{bmatrix}2&-1.8\\-1.8&2 \end{bmatrix}\Bigg)
\end{align*}

\subsubsection{Runtime Comparisons} \label{sec:append:runtime}
We report runtime comparisons between AMAGOLD and SGHMC on Dist1 and Dist2 with step size 0.15 (Figure \ref{fig:runtimeapp}). This experiment uses the analytical energy expression (no data examples), so there is no speed-up of stochastic methods over full-batch methods. At the beginning, SGHMC converges faster due to the lack of M-H step, but eventually it converges to a biased distribution. AMAGOLD is not much slower than SGHMC, which shows that AMA can reduce the amount of computation of adding M-H step while keep the chain unbiased.

\begin{figure*}[h!]
    \centering
    \vspace{-2mm}
    \begin{tabular}{cc}		
    	\includegraphics[width=6cm]{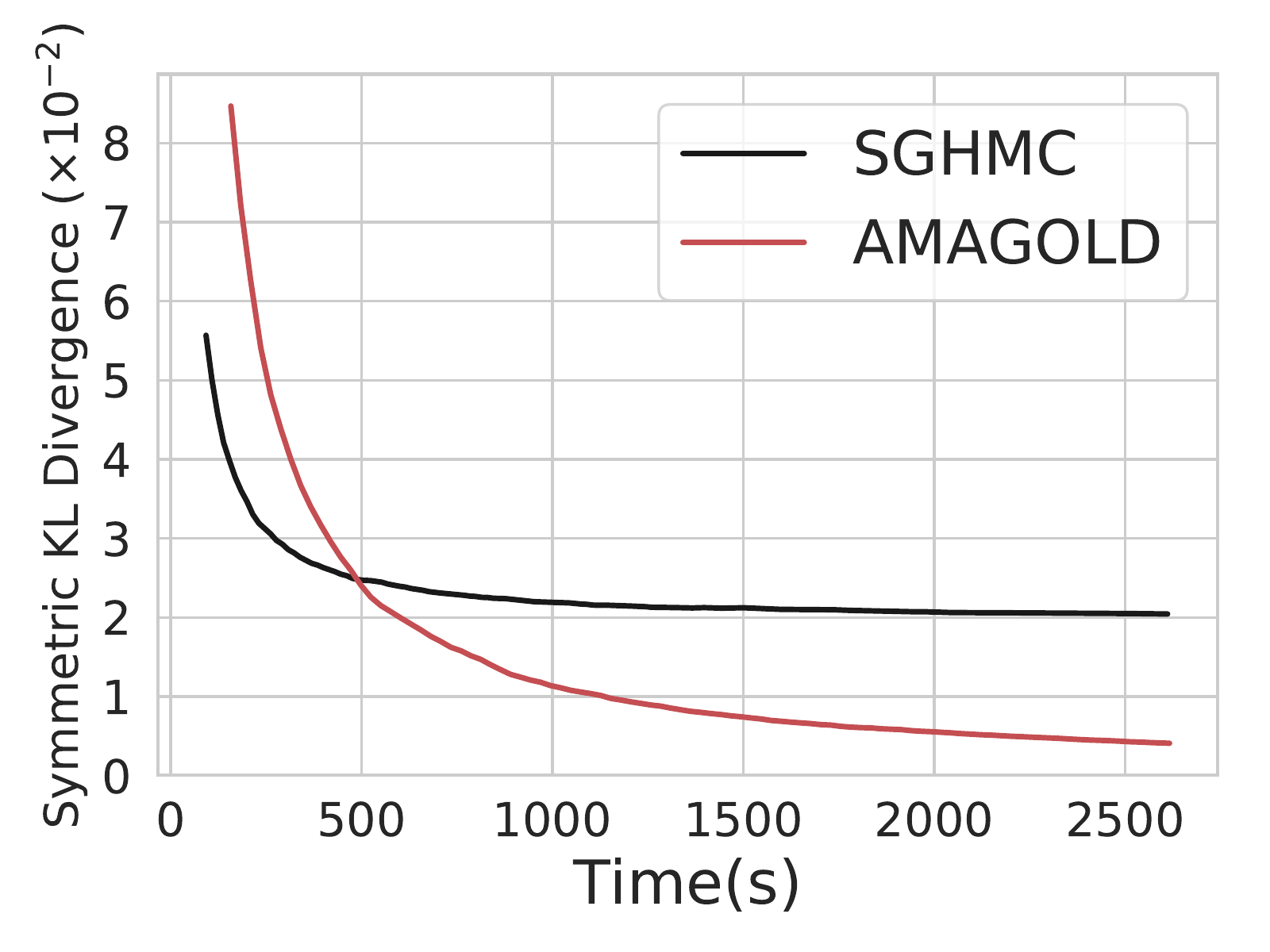}  &
    	\includegraphics[width=6cm]{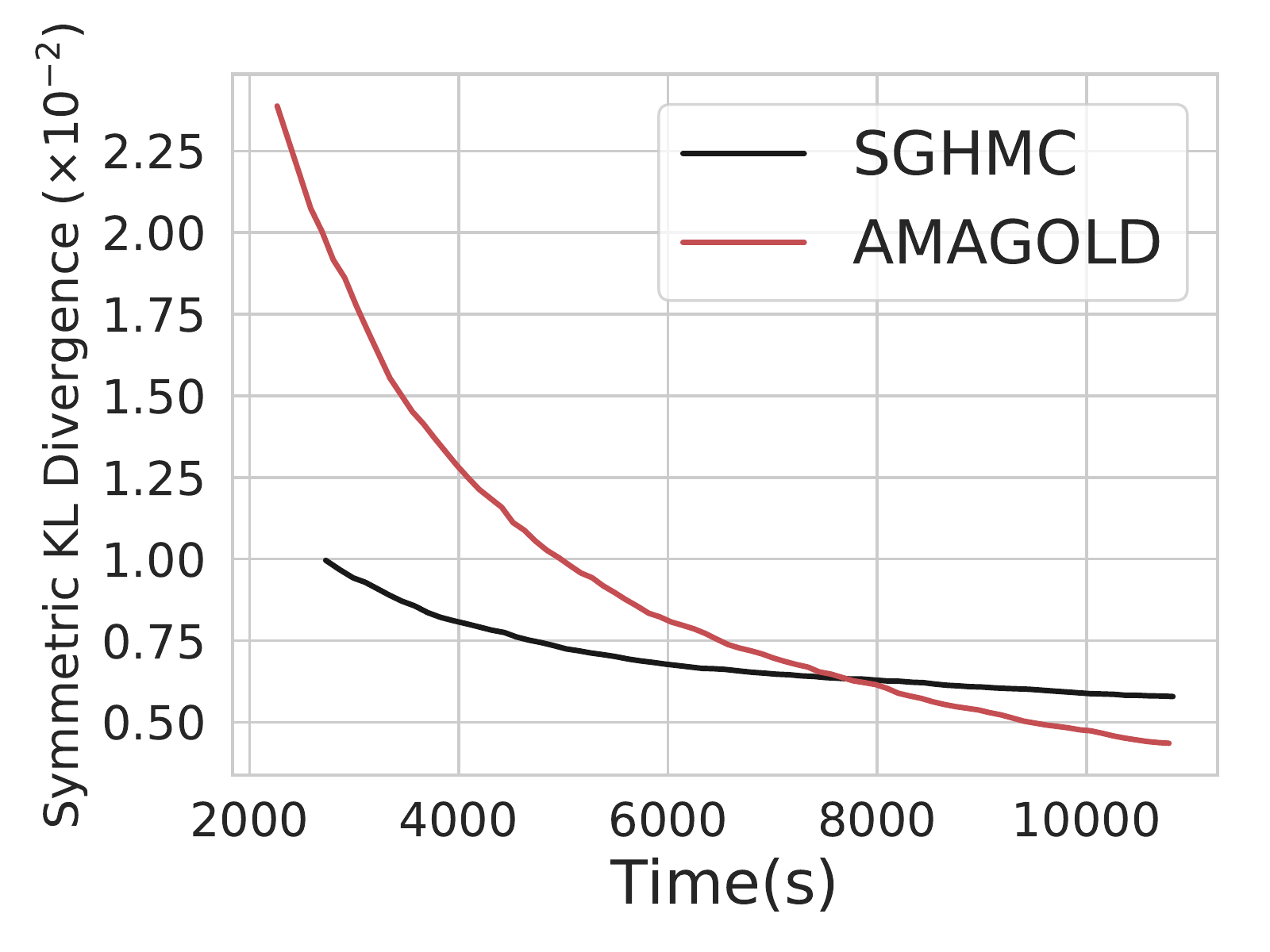} 
    	\\		
    	(a) &
    	(b)
    \\		
    \end{tabular}
    \caption{Runtime comparisons between SGHMC and AMAGOLD on synthetic distributions (a) Dist1 and (b) Dist2.}
    \label{fig:runtimeapp}
\end{figure*}
\vspace{-2mm} 

\subsubsection{Additional Note on Figure~\ref{fig:synthetic}} \label{sec:append:kl}
It is worth noting that, even though it is lower than SGHMC's, AMAGOLD's KL divergence grows when the step size is large compared to full-batch methods. This is because the M-H acceptance probability decreases, causing the chain to converge more slowly. This is expected. It is well-known that stochastic methods are more sensitive to step sizes than full-batch methods \citep{nemirovski2009robust}. However, since AMAGOLD's KL divergence grows much slower than SGHMC's, AMAGOLD is more robust to different step sizes

\subsection{Bayesian Logistic Regression} \label{sec:append:lr}
We report the acceptance probability of AMAGOLD on \emph{Heart} for varying step sizes in Figure \ref{fig:mhheart}. For a large range of step sizes, the acceptance rate is sufficiently high to allow the chain converge fast, demonstrated in Figure \ref{fig:rw}. The acceptance rate may become very low with a large step size resulting in slow move. But this undesired acceptance probability can be easily detected and avoided in practice.

\begin{figure}[h]
    \vspace{-0mm}
        \hspace{-5mm}
        \centering
        \includegraphics[width=6cm]{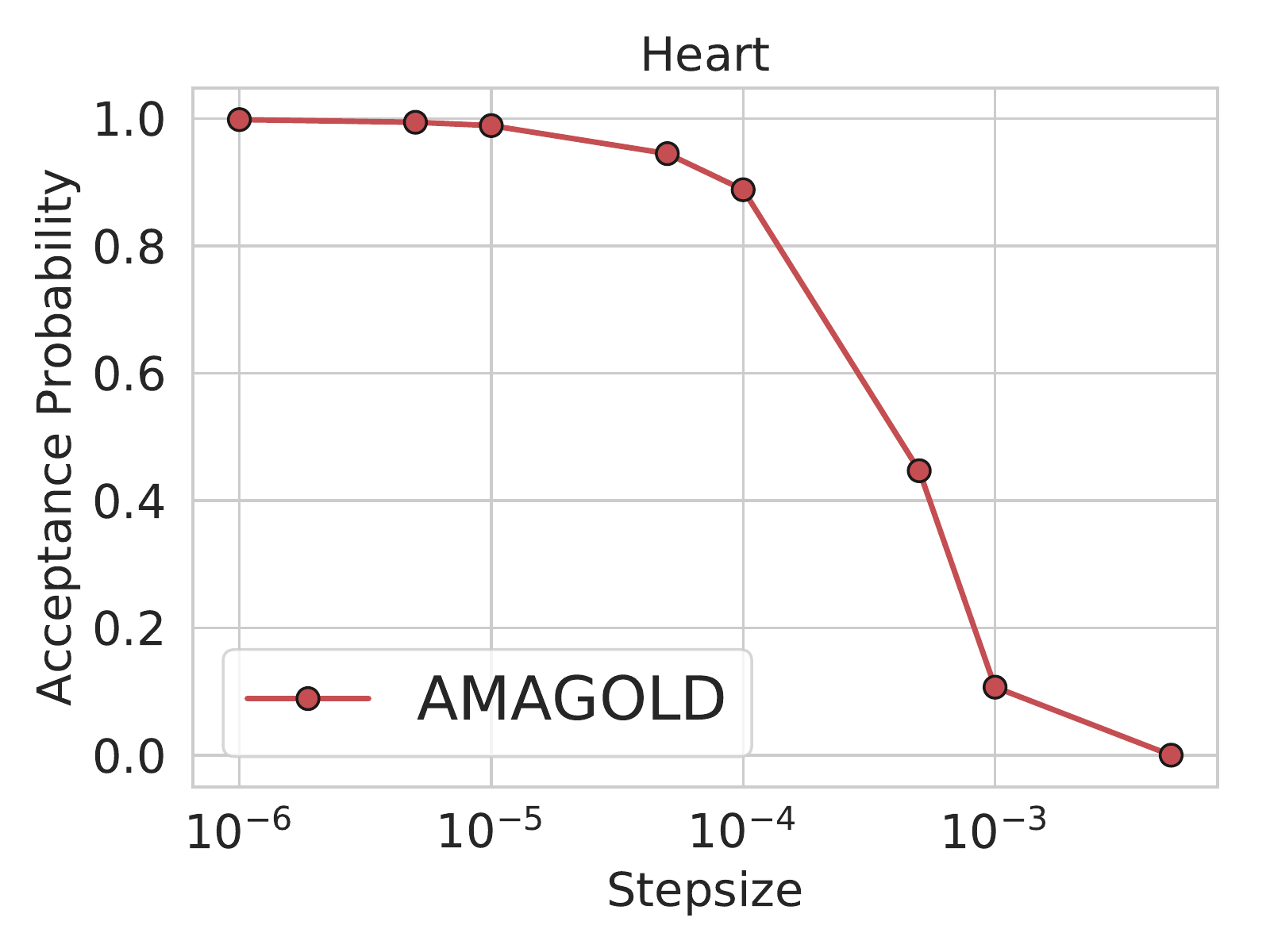}
        \hspace{-4mm}
    \vspace{-0mm}
    \caption{The acceptance probability of the M-H step in AMAGOLD for varying step sizes on the Heart dataset.}
    \label{fig:mhheart}
    \vspace{-0mm}
\end{figure}

\subsection{Bayesian Neural Networks}
The architecture of Bayesian Neural Networks is a two-layer MLP with first hidden layer size 500 and the second hidden layer size 256. 

\end{document}

%% file: mlVecMat.tex
\usepackage{amsmath}
\usepackage{amsfonts}
\usepackage{amssymb}
\usepackage{wrapfig}
\usepackage{subcaption}
\usepackage{multirow}
 \usepackage{mathtools} 

\usepackage{verbatim}

\usepackage{anyfontsize}

\usepackage{color}
\usepackage{tikz}
\usepackage{breqn}
\usetikzlibrary{arrows,shapes,snakes,automata,backgrounds,fit,petri}
\usepackage{adjustbox}

\newcommand{\RN}[1]{%
	\textup{\lowercase\expandafter{\it \romannumeral#1}}%
}

\makeatletter
\newcommand{\distas}[1]{\mathbin{\overset{#1}{\kern\z@\sim}}}%

\usepackage{enumitem}

\usepackage{algorithm}
\usepackage{algorithmic}

\newcommand{\beq}{\vspace{0mm}\begin{equation}}
\newcommand{\eeq}{\vspace{0mm}\end{equation}}
\newcommand{\beqs}{\vspace{0mm}\begin{eqnarray}}
\newcommand{\eeqs}{\vspace{0mm}\end{eqnarray}}
\newcommand{\barr}{\begin{array}}
\newcommand{\earr}{\end{array}}

\newtheorem{theorem}{Theorem} 
\newtheorem{lemma}{Lemma}

\ifx\assumption\undefined
\newtheorem{assumption}{Assumption}
\fi

\ifx\definition\undefined
\newtheorem{definition}{Definition}
\fi

\ifx\remark\undefined
\newtheorem{remark}{Remark}
\fi
\newcommand{\norm}[1]{\left\| #1 \right\|}

\newcommand{\Abs}[1]{\left| #1 \right| }

\newcommand{\Probc}[3][]{\mathbf{P}_{#1}\left( \hiderel{#2} \middle | \hiderel{#3} \right) }
\newcommand{\Exv}[2][]{\mathbf{E}_{#1}\left[ #2 \right]}

\newcommand{\trace}[1]{\mathbf{tr}\left( #1 \right)}